\documentclass[letterpaper]{article} 
\usepackage[preprint]{aaai2027}  
\usepackage[hyphens]{url}  
\usepackage{graphicx} 
\usepackage{natbib}  
\usepackage{caption} 
\usepackage{mathtools} 
\usepackage{amsthm} 
\usepackage{booktabs} 
\usepackage{tikz} 

\usepackage[utf8]{inputenc} 
\usepackage[T1]{fontenc}    
\usepackage{algorithm}
\usepackage{algpseudocode}
\usepackage{url}            
\usepackage{booktabs}       
\usepackage{amsfonts}       
\usepackage{nicefrac}       
\usepackage{microtype}      
\usepackage{xcolor}         
\usepackage{romannum}
\usepackage{subcaption}
\usepackage{graphicx}
\usepackage{amsthm,amssymb,mathtools}
\newtheorem{theorem}{Theorem}
\newtheorem{lemma}[theorem]{Lemma}
\newtheorem{assumption}{Assumption}

\theoremstyle{remark}
\newtheorem{remark}{Remark}
\usepackage{amsmath}

\usepackage{newfloat}
\usepackage{listings}
\DeclareCaptionStyle{ruled}{labelfont=normalfont,labelsep=colon,strut=off} 
\floatstyle{ruled}
\newfloat{listing}{tb}{lst}{}
\floatname{listing}{Listing}

\usepackage{booktabs}

\title{Active Preference Learning over Latent Preference Archetypes for Many-Objective Bayesian Optimization}
\author{
    Manisha Dubey \corresponding \textsuperscript{\rm 1,2}\thanks{This work was carried out while the author was at University of Manchester, UK}\\
    Sebastiaan De Peuter\textsuperscript{\rm 3,4}\thanks{This work was carried out while the author was at Aalto University, Finland}\\
    Wanrong Wang\textsuperscript{\rm 1},
    Samuel Kaski\textsuperscript{\rm 1,3,5}
}
\affiliations{
      \textsuperscript{\rm 1}    Department of Computer Science, University of Manchester, UK\\
    \textsuperscript{\rm 2}    Centre for AI in Assistive Autonomy, University of Edinburgh, Scotland, UK\\
     \textsuperscript{\rm 3} Department of Computer Science, Aalto University, Espoo, Finland\\
     \textsuperscript{\rm 4}  Informatics Institute, University of Amsterdam, Amsterdam, Netherlands\\
     \textsuperscript{\rm 5} ELLIS Institute Finland, Helsinki, Finland\\
  }

\begin{document}

\maketitle

\begin{abstract}
    Preference-based many-objective Bayesian optimization typically assumes that all pairwise comparisons arise from a single latent utility function, despite real decision makers often exhibiting multiple latent preference archetypes across contexts. We propose an active preference learning framework for many-objective Bayesian optimization that infers latent preference archetypes from pairwise comparisons, enabling both the identification of the active trade-off strategy and the refinement of its associated preferences. Our framework represents preferences as a Dirichlet-process mixture of latent archetypes and introduces mixture-aware information-theoretic query strategies that separately target archetype identification and within-archetype refinement through a hybrid acquisition policy. Experiments on synthetic benchmarks and real-world chemical process design case-study consistently outperforms state-of-the-art preference-based Bayesian optimization methods while recovering interpretable latent preference structure beyond conventional single-utility models. The proposed mixture-aware diagnostics further quantify archetype recovery and preference calibration, providing insights that are not captured by optimization performance alone.
\end{abstract}


\section{Introduction}

Real-world applications such as materials and additive manufacturing~\cite{hastings2025accelerated,myung2025multi}, power systems \cite{ah2005evolutionary}, robotics \cite{li2023online,shen2025real}, urban design \cite{liu2023multi}, and recommender systems \cite{jannach2023survey} require optimizing many competing objectives \cite{deb2016multi} under expensive evaluations. Bayesian optimization provides an efficient framework for sequentially identifying high-quality solutions in such settings. As the number of objectives grows (\(L\ge4\)), however, many-objective Bayesian optimization (MaOBO) becomes increasingly challenging due to rapidly expanding Pareto frontiers, weakened dominance relations, and the difficulty of selecting solutions that reflect a decision maker's priorities \cite{ishibuchi2008evolutionary,binois2020kalai,zhang2020random}. While optimization algorithms can efficiently approximate large Pareto sets, the fundamental challenge increasingly shifts from generating candidate solutions to identifying which trade-offs are actually preferred.

Preference-based Bayesian optimization addresses this challenge by incorporating pairwise preference feedback to guide the search towards Pareto-optimal solutions that align with the decision maker's objectives, avoiding the need to approximate the entire Pareto frontier. Existing methods have demonstrated that actively eliciting preferences can substantially improve optimization efficiency \cite{astudillo2020multi,ozaki2024multi,xu2024principled,av2022human}. However, they almost universally assume that all observed comparisons are generated by a single latent utility function. This assumption becomes restrictive in many-objective settings, where users often exhibit latent preference archetypes, for example, prioritizing cost, safety, or sustainability under different circumstances. Compressing these behaviours into a single utility function can blur preference estimates, reduce query efficiency, and ultimately misdirect optimisation towards less relevant regions of the Pareto front.

We address this limitation by introducing a mixture-aware active preference learning framework for many-objective Bayesian optimization that represents heterogeneous decision-maker preferences as latent preference archetypes rather than a single utility function. The inferred posterior over archetypes naturally decomposes preference elicitation into two complementary objectives: identifying the active archetype and refining its associated trade-offs. This leads to inter-archetype, intra-archetype, and hybrid information-theoretic query strategies that actively select informative pairwise comparisons. Combined with many-objective Bayesian optimization, the learned preference posterior guides expensive evaluations towards decision-maker-preferred regions of the Pareto front, enabling more efficient preference-guided optimization. Our contributions can be summarized as follows:

\begin{itemize}

\item We introduce a mixture-aware active preference learning framework for many-objective Bayesian optimization that models heterogeneous decision-maker preferences as latent preference archetypes using a Dirichlet-process mixture, relaxing the single-utility assumption.

\item We develop inter-archetype, intra-archetype, and hybrid information-theoretic query strategies that actively identify latent preference archetypes and refine within-archetype trade-offs from pairwise comparisons.

\item We integrate the inferred preference posterior into the Bayesian optimization loop to guide expensive evaluations towards decision-maker-preferred regions of the Pareto front, and introduce mixture-aware diagnostics for analysing archetype recovery and preference calibration.

\item We demonstrate consistent improvements over state-of-the-art preference-based Bayesian optimization methods on synthetic benchmarks and a real-world chemical process design case study, while recovering interpretable latent preference archetypes beyond conventional single-utility models.

\end{itemize}

\section{Problem Formulation}

We consider many-objective Bayesian optimization (MaO-BO) with \(L \ge 4\) objectives.
Let \(\mathcal X \subset \mathbb R^d\) be a design space and
\(\mathbf f:\mathcal X \to \mathbb R^L\) an unknown vector-valued objective.
Querying a design \(x \in \mathcal X\) returns a (possibly noisy) observation
$
\mathbf y = \mathbf f(x) + \boldsymbol\epsilon$
where $\boldsymbol\epsilon \sim \mathcal N(0,\sigma^2 I)
$
and evaluations are assumed expensive. Many-objective problems typically admit a large set of Pareto-optimal solutions.
In human-in-the-loop settings, however, the goal is not to recover the entire Pareto front,
but to identify solutions that are most desirable to a decision maker (DM). Direct evaluations of the DM's utility are unavailable.
Instead, the learner may query pairwise preferences between previously observed outcomes. Given two outcomes \(\mathbf y_i\) and \(\mathbf y_i'\), 
the decision maker returns a binary comparison 
indicating whether \(\mathbf y_i \succ \mathbf y_i'\). Further, we assume that the DM evaluates outcomes through an unknown utility function.
To model heterogeneous trade-offs, we posit that the DM's utility
is governed by one of \(K\) latent preference modes.
Each mode \(k \in [K]\) is parameterized by a weight vector
\(\mathbf w_k \in \Delta^{L-1}\),
and modes occur with unknown mixture weights
\(\boldsymbol\eta \in \Delta^{K-1}\). This formulation captures either a single DM with context-dependent trade-offs
or a heterogeneous population of DMs.
We denote the scalar utility induced by mode \(k\) as
\(U(\mathbf f(x);\mathbf w_k)\). Our objective is to identify designs that maximize the DM's mixture-expected utility,
\setlength{\abovedisplayskip}{1pt}
\setlength{\belowdisplayskip}{1pt}
\setlength{\abovedisplayshortskip}{1pt}
\setlength{\belowdisplayshortskip}{1pt}
\begin{equation}
x^\star \in \arg\max_{x \in \mathcal X}
\sum_{k=1}^K \eta_k
\, \mathbb E\!\left[ U(\mathbf f(\mathbf{x}); \mathbf w_k) \right],
\end{equation}
where the expectation is taken over observation noise.
The learner must therefore infer the latent mixture parameters
\(\{\eta_k,\mathbf w_k\}_{k=1}^K\) from pairwise feedback
while sequentially selecting designs to evaluate. 

\begin{figure*}[h!]
    \centering
    \includegraphics[width=0.79\textwidth,
    height=0.73\textheight,
    keepaspectratio]{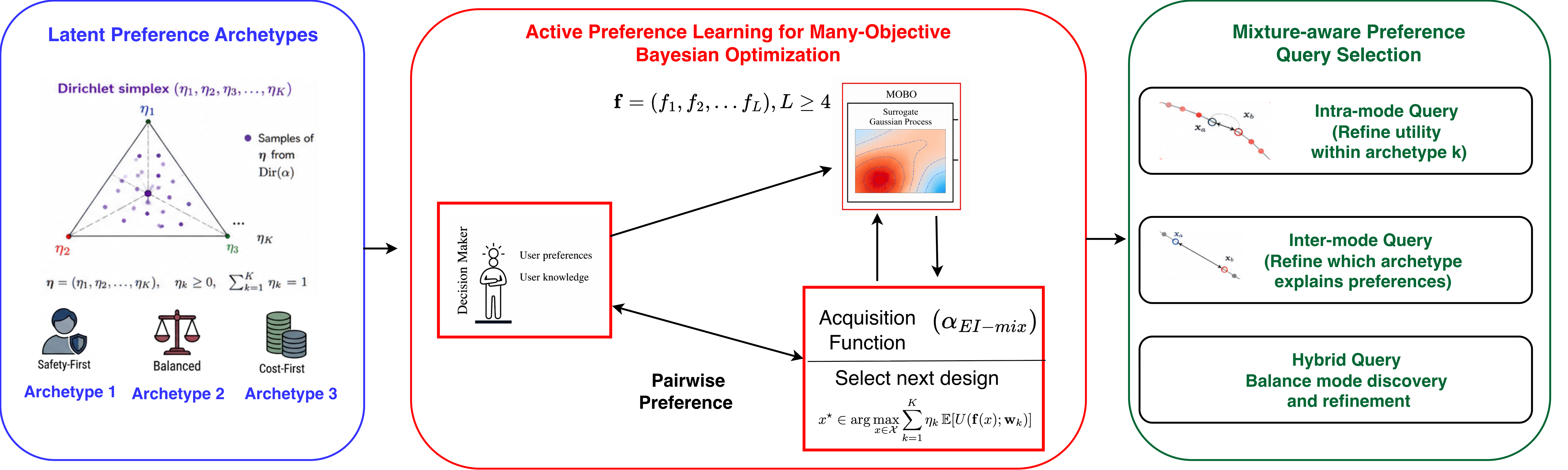}
    \caption{Overview of the proposed active preference learning framework for many-objective Bayesian optimization. Pairwise comparisons are used to infer latent preference archetypes and their mixture weights(left). The inferred posterior guides Gaussian-process-based many-objective Bayesian optimization through a mixture-aware acquisition function and iterative preference feedback from decision maker (centre). Active query selection chooses informative intra-mode, inter-mode, and hybrid comparisons to jointly refine utility estimation and latent preference archetype discovery (right).}
    \label{fig:architecture}
\end{figure*}

\section{Proposed Methodology}
 In this section, we propose a preference-based many-objective Bayesian optimization framework that models decision-maker preferences as a mixture of latent preference archetypes, relaxing the conventional single-utility assumption (see Fig. \ref{fig:architecture}). Modeling preferences as a mixture of latent preference archetypes allows multiple preference archetypes to coexist. This introduces two complementary sources of uncertainty: identifying the active preference archetype and learning its associated trade-offs. Accordingly, our framework consists of four components: (i) Gaussian process surrogates for the expensive objectives, (ii) a latent mixture preference model, (iii) variational posterior inference, and (iv) mixture-aware acquisition strategies for selecting both new evaluations and informative preference queries. We also present the algorithm \ref{alg:maobo-mix} here.

\subsection{Objective Surrogate Model}
We model each objective with an independent Gaussian process (GP). Let
$\mathbf{f}:\mathcal{X}\to\mathbb{R}^{L}$ denote the $L$ objectives. Each objective function \(f_\ell\) is modeled independently with a Gaussian process:
$ 
f_\ell \sim \mathcal{GP}(m_\ell, k_\ell), \quad \ell = 1,\dots,L.
$ Conditioning on observed data yields
predictive means and variances. Given observations \(\mathcal D_f = \{(x_i,\mathbf y_i)\}\),
we obtain posterior predictive distributions
$
f_\ell(x) \mid \mathcal D_f
\sim
\mathcal N(\mu_\ell(x), \sigma_\ell^2(x)).
$
Independence across objectives keeps inference
and acquisition scalable in $L$ \cite{rasmussen2010gaussian, williams2006gaussian}. One could also employ multi-output gaussian process \cite{alvarez2008sparse}.

\subsection{Latent Preference Archetypes}

Existing preference-based Bayesian optimization methods typically assume that all pairwise comparisons arise from a single latent utility function \cite{ozaki2023multi, astudillo2020multi}. We instead represent decision-maker preferences as a mixture of \(K\) latent preference archetypes, allowing multiple trade-off strategies to coexist. Unlike a single utility function, the mixture representation allows different pairwise comparisons to be explained by different latent preference archetypes, naturally capturing heterogeneous or context-dependent trade-off strategies. Each archetype corresponds to a distinct scalarization of the objectives, while the mixture captures uncertainty over which archetype generated each observed comparison. Each archetype \(k\) is parameterized by a weight vector
\(
\mathbf w_k\in\Delta^{L-1}
\)
that defines a Chebyshev utility over objective outcomes.
For a predicted outcome \(\mathbf y\), the utility under archetype \(k\) is

\begin{equation}
U(\mathbf y;\mathbf w_k)
=
-
\min_{\ell}
\frac{y_\ell}{w_{k\ell}},
\qquad
w_{k\ell}>0,
\end{equation}

where utilities are computed using the GP posterior predictive means. Rather than fixing a single preference model, we place a Dirichlet-process mixture prior over the archetypes. Using a truncated stick-breaking construction~\cite{blei2006variational}, $ v_k\sim\mathrm{Beta}(1,\alpha), \eta_k = v_k \prod_{j<k}(1-v_j),$  yielding mixture weights $\boldsymbol\eta\in\Delta^{K-1}$. Each archetype is assigned a preference weight vector $ \mathbf w_k \sim \mathrm{Dir}(\boldsymbol\beta)$ representing a distinct trade-off over the objectives. For each observed comparison
\(
(\mathbf y_i,\mathbf y_i')
\),
a latent archetype assignment $ z_i \sim \mathrm{Categorical}(\boldsymbol\eta) $ selects the preference mode responsible for that comparison. Conditioned on \(z_i=k\), pairwise preferences follow the probit random-utility model,
\setlength{\abovedisplayskip}{0pt}
\setlength{\belowdisplayskip}{0pt}
\setlength{\abovedisplayshortskip}{0pt}
\setlength{\belowdisplayshortskip}{0pt}
\begin{equation}
\label{eq:probit_likelihood}
P(\mathbf y_i \succ \mathbf y'_i \mid z_i = k)
=
\Phi\!\left(
\frac{
U(\mathbf y_i;\mathbf w_k)
-
U(\mathbf y'_i;\mathbf w_k)
}{
\sqrt{2}\sigma_u
}
\right).
\end{equation}

Then we marginalize latent assignments as
\setlength{\abovedisplayskip}{0pt}
\setlength{\belowdisplayskip}{0pt}
\setlength{\abovedisplayshortskip}{0pt}
\setlength{\belowdisplayshortskip}{0pt}
\begin{equation}
p(\mathcal D_{\mathrm{pref}} \mid \boldsymbol\eta, \mathbf w_{1:K})
\!\!=\!\!
\prod_{i=1}^N
\sum_{k=1}^K
\eta_k
\Phi\!\left(
\frac{
U(\mathbf y_i;\mathbf w_k)
-
U(\mathbf y'_i;\mathbf w_k)
}{
\sqrt{2}\sigma_u
}
\right)
\end{equation}

We adopt the standard probit random-utility model for pairwise comparisons \cite{astudillo2020multi, ozaki2023multi}. The proposed framework is independent of this choice and can be combined with alternative pairwise likelihoods such as Bradley-Terry-Luce.

\subsection{Posterior Inference}
\label{sec:inference}
Let
\(
\mathcal D_{\mathrm{pref}}
=
\{(\mathbf y_i,\mathbf y_i')\}_{i=1}^N
\)
denote the observed pairwise comparisons. Given the generative model, our objective is to infer the posterior over the latent preference variables, $p(\boldsymbol\eta,\mathbf w_{1:K},\mathbf z,\alpha \mid \mathcal D_{\mathrm{pref}}) $ where \(\boldsymbol\eta\), \(\mathbf w_{1:K}\), \(\mathbf z\), and \(\alpha\) denote the mixture weights, preference archetype weights, latent archetype assignments, and concentration parameter, respectively. The corresponding joint distribution is
\setlength{\abovedisplayskip}{0pt}
\setlength{\belowdisplayskip}{0pt}
\setlength{\abovedisplayshortskip}{0pt}
\setlength{\belowdisplayshortskip}{0pt}
\begin{align}
&p(\mathcal D_{\mathrm{pref}}, \mathbf z, \boldsymbol\eta, \mathbf w_{1:K}, \alpha)
 \\
&=
p(\alpha) \!\!
\prod_{k=1}^K p(v_k \!\! \mid \!\! \alpha)\, p(\mathbf w_k)
\prod_{i=1}^N
p(z_i \! \mid \! \boldsymbol\eta)\,
p(\mathbf y_i \! \succ \! \mathbf y'_i \mid \! z_i, \mathbf w_{z_i}) \notag
\end{align}
Exact posterior inference is intractable because the latent archetype assignments are marginalized and the probit likelihood is non-conjugate. We therefore approximate the posterior using mean-field stochastic variational inference. Under the latent mixture preference model, the unknown variables consist of the mixture weights
\(\boldsymbol\eta\),
archetype weights
\(\mathbf w_{1:K}\),
latent archetype assignments
\(\mathbf z\),
and concentration parameter
\(\alpha\).
The resulting posterior $p(\boldsymbol\eta,\mathbf w_{1:K},\mathbf z,\alpha \mid \mathcal D_{\mathrm{pref}}) $ is analytically intractable because of the latent mixture structure and non-conjugate preference likelihood. We therefore approximate the posterior using a mean-field variational distribution
\begin{equation}
q(\boldsymbol\eta,\mathbf w_{1:K},\mathbf z,\alpha)
=
q(\alpha)
\prod_{k=1}^K
q(v_k)\,
q(\mathbf w_k)
\prod_{i=1}^N
q(z_i).
\end{equation}

The variational parameters are obtained by maximizing the evidence lower bound (ELBO)
\begin{equation}
\mathcal L(q)
=
\mathbb E_q\!\left[
\log p(\mathcal D_{\mathrm{pref}}, \mathbf z, \boldsymbol\eta, \mathbf w_{1:K}, \alpha)
\right]
-
\mathbb E_q\!\left[
\log q(\boldsymbol\eta, \mathbf w_{1:K}, \mathbf z, \alpha)
\right]
\end{equation}
using stochastic variational inference (SVI).

\subsection{Design Acquisition}

At iteration \(t\), the learner maintains a posterior over the objective functions and latent preference parameters. Let $ \hat{\theta} = (\hat{\boldsymbol\eta},\hat{\mathbf w}_{1:K}) $ denote the posterior mean estimates of the mixture weights and preference archetype weights obtained from the variational posterior. The GP surrogate induces a posterior
\(
p(\mathbf f(x)\mid\mathcal D_t)
\)
over the objective vector at any candidate design \(x\). For a given objective vector \(\mathbf y\), the mixture utility is computed as $
U_{\mathrm{mix}}(\mathbf y;\hat{\theta})
=
\sum_{k=1}^{K}
\hat{\eta}_k\,
U(\mathbf y;\hat{\mathbf w}_k),
$ which represents the expected utility under the inferred mixture of preference archetypes. The next design is selected by maximizing the mixture expected improvement (mixture-EI),
\begin{equation}
\label{eq:aqn_fn}
\alpha_{\mathrm{EI\mbox{-}mix}}(x)
=
\mathbb E_{\mathbf f(x)}
\!\left[
\max\!\left(
U_{\mathrm{mix}}(\mathbf f(x);\hat{\theta})
-
U_{\mathrm{best}},
0
\right)
\right],
\end{equation}
where $ U_{\mathrm{best}} = \max_{i\le t} U_{\mathrm{mix}}(\mathbf y_i;\hat{\theta})$
is the best observed mixture utility under the inferred preference model. The expectation is taken with respect to the GP posterior and is approximated using Monte Carlo samples of \(\mathbf f(x)\). This acquisition balances exploration and exploitation while incorporating the inferred mixture of latent preference archetypes through the posterior mean preference estimates. Using posterior mean estimates provides a computationally efficient plug-in approximation to the fully Bayesian acquisition, avoiding nested Monte Carlo integration over both objective and preference uncertainty.
\subsection{Preference Query Acquisition}
\label{sec:preference_query_acquisition}
In addition to evaluating new designs, the learner may select a
pairwise comparison between two previously observed outcomes
$(\mathbf y_i,\mathbf y_j)$. Let
$ r \in \{0,1\} $ denote the random response of the decision maker (DM),
where $r=1$ indicates $\mathbf y_i \succ \mathbf y_j$. Let
\(
\theta=(\boldsymbol\eta,\mathbf w_{1:K})
\)
denote the latent preference parameters, and let
\(q_t(\theta)\)
denote the current variational posterior. The predictive probability of the response is 
$p(r=1\mid y_i,y_j,\mathcal D_t)
=
\int
p(r=1\mid y_i,y_j,\theta)
q_t(\theta)
d\theta$. The DM's preference $r$ updates the posterior $p(\theta \mid \mathcal D_t)$,
which in turn influences future design decisions. The value of a preference query can therefore be interpreted as its expected reduction in uncertainty about $\theta$, and hence its expected improvement in downstream decision quality. Computing the exact expected value of a preference query is generally intractable, as it requires updating the posterior for every possible response. Following \cite{ozaki2023multi}, we therefore adopt a Bayesian Active Learning by Disagreement (BALD) criterion \cite{houlsby2011bayesian}, which approximates query value by the mutual information between the response $r$ and the latent preference parameters $\theta$. 

\begin{equation}
\alpha_{\mathrm{pref}}(i,j)
=
I(r;\theta \mid \mathbf y_i,\mathbf y_j,\mathcal D_t)
\end{equation}

Using the entropy decomposition of mutual information, this can be written as
\begin{equation}
I(r;\theta)
=
H[r]
-
\mathbb E_{\theta \sim p(\theta \mid \mathcal D_t)}
\left[
H[r \mid \theta]
\right]
\end{equation}

where $H[\cdot]$ denotes Bernoulli entropy.
The first term measures overall predictive uncertainty,
while the second term measures expected uncertainty
under a fixed parameter setting.
Their difference quantifies how much the response
would reduce posterior uncertainty. The conditional preference probability
given $\theta$ is
\[
p(r=1 \mid \theta)
=
\sum_{k=1}^K \eta_k
\,
\Phi\!\left(
\frac{
U(\mathbf y_i;\mathbf w_k)
-
U(\mathbf y_j;\mathbf w_k)
}{\sqrt{2}\sigma_u}
\right)
\]

The marginal predictive probability is approximated by Monte Carlo samples from the variational posterior \(q_t(\theta)\). Under the proposed mixture model, preference queries naturally decompose into two complementary objectives: reducing uncertainty over the active preference archetype (inter-cluster queries) and reducing uncertainty over the preference weights within an archetype (intra-cluster queries). A hybrid acquisition combines these objectives through a convex combination. Please note that other information-theoretic criteria (e.g., predictive entropy) could be substituted within the same framework.

\paragraph{Inter-archetype query}
In the mixture preference model, uncertainty arises both from
which preference archetype governs the DM and from the weights within each mode.
The inter-cluster acquisition specifically targets uncertainty
about the latent preference archetype. Let $K$ denote the discrete latent archetype variable and
$r \in \{0,1\}$ the response to a comparison $(\mathbf y_i,\mathbf y_j)$.
We define the inter-cluster acquisition as the mutual information
between the response and the archetype identity, which can be further decomposed as:
\setlength{\abovedisplayskip}{1pt}
\setlength{\belowdisplayskip}{1pt}
\setlength{\abovedisplayshortskip}{1pt}
\setlength{\belowdisplayshortskip}{1pt}
\begin{equation}
\label{eq:inter_query}
I(r;K) = H[r]- \mathbb E_{K}\!\left[H[r \mid K]\right].
\end{equation}
\setlength{\abovedisplayskip}{1pt}
\setlength{\belowdisplayskip}{1pt}
\setlength{\abovedisplayshortskip}{1pt}
\setlength{\belowdisplayshortskip}{1pt}
The entropy form becomes
\begin{equation}
A_{inter}
=
H[p_{\mathrm{mix}}]
-
\sum_{k=1}^K \eta_k H[p_k].
\end{equation}

where mixture predictive is $
p_{\mathrm{mix}} = \sum_{k=1}^K \eta_k p_k$. 
This quantity is large when the mixture predictive is uncertain
(high $H[p_{\mathrm{mix}}]$)
but each individual cluster prediction is confident.

\paragraph{Intra-archetype query}
While the inter-mode criterion targets uncertainty over the active archetype,
the intra-mode criterion targets uncertainty over the archetype weights
within a particular preference archetype. For a candidate comparison between two previously observed outcomes $(\mathbf y_i, \mathbf y_j)$ with response $r \in \{0,1\}$ and target archetype $c \in [K]$,we define the intra-mode acquisition as the mutual information between $r$ and the archetype-$c$ weight vector:
\setlength{\abovedisplayskip}{1pt}
\setlength{\belowdisplayskip}{1pt}
\setlength{\abovedisplayshortskip}{1pt}
\setlength{\belowdisplayshortskip}{1pt}
\begin{equation}
\label{eq:intra_query}
A_{\mathrm{intra}}^{(c)}(i,j)
=
I\!\left(r;\mathbf w_c \mid K=c,\,\mathbf y_i,\mathbf y_j,\,\mathcal D_t\right)
\end{equation}

Using the BALD entropy decomposition, we obtain
\[
A_{\mathrm{intra}}^{(c)}(i,j)
=
H\!\left(\bar p_c\right)
-
\mathbb E_{\mathbf w \sim q_t(\mathbf w_c)}
\left[
H\!\left(p_{\mathbf w}\right)
\right]
\]
where $ \bar p_c
= \mathbb E_{\mathbf w \sim q_t(\mathbf w_c)}\!\left[p_{\mathbf w}\right]$ and 
$q_t(\mathbf w_c)$ denotes the current approximate posterior
(variational factor) over $\mathbf w_c$ after observing $\mathcal D_t$. 

\paragraph{Hybrid query} The hybrid acquisition balances two complementary objectives:
identifying the active preference archetype (inter-cluster exploration)
and refining the preference weights within that archetype
(intra-cluster refinement).
\begin{equation}
\label{eq:hybrid_query}
A_{\mathrm{hybrid}}
=
\lambda A_{\mathrm{inter}}
+
(1-\lambda) A_{\mathrm{intra}},
\quad \lambda \in [0,1].
\end{equation}

The hyperparameter $\lambda$
controls the exploration trade-off between
mode disambiguation and weight refinement.

Under exact GP inference, the dominant per-iteration computational cost arises from fitting $L$ independent Gaussian process (GP) surrogates over $n$ observations, resulting in a complexity of $\mathcal{O}(Ln^3)$. The additional cost of variational mixture inference scales approximately linearly with the number of mixture components $K$ and the number of preference observations $q$, giving $\mathcal{O}(Kq)$ (or $\mathcal{O}(Kn)$ when $q$ scales with $n$). Consequently, a practical upper bound on the per-iteration computational complexity is $\mathcal{O}(Ln^3 + Kq)$ where GP updates remain the dominant computational bottleneck, while the overhead introduced by the mixture model is approximately linear in the number of active components.

\begin{algorithm}[t]
\caption{\textsc{Mixture-Based MaOBO}}
\label{alg:maobo-mix}
\begin{algorithmic}[1]
\Require $\mathbf{f}:\mathcal X \to \mathbb R^L$, truncation $K$, query mode
\State Initialize dataset $\mathcal D_f$ and preference set $\mathcal D_{pref}$
\For{$t=1,2,\dots$}
    \State Fit $L$ independent GPs on $\mathcal D_f$
    \State Update DP-mixture posterior via SVI to obtain  $(\hat{\eta}, \{\hat w_k\}_{k=1}^K)$
    \State $\mathbf{x}_t \gets \arg\max_{\mathbf{x}\in\mathcal X}\textsf{EI}_{\textsf{mix}}(\mathbf{x};\hat\eta,\{\hat{\mathbf{w}_k}\})$  (using Eq. \ref{eq:aqn_fn})
    \State Evaluate $\mathbf{y}_t = \mathbf{f}(\mathbf{x}_t)$ and update $\mathcal D_f$
    \State $(\mathbf{y}_a, \mathbf{y}_b) \gets \textsc{SelectPair}(\{\mathbf{y}_i\}, \hat{\eta}, \{\hat{\mathbf{w}}_k\}, \textsf{mode})$ 
    \State Query for preference label $\mathbf{y}_{\text{pref}}\succ \mathbf{y}_{\text{other}}$ 
    \State Set $\mathcal D_{pref}\leftarrow \mathcal D_{pref}\cup\{(\mathbf{y}_{\text{pref}}, \mathbf{y}_{\text{rej}})\}$ 
\EndFor
\end{algorithmic}
\end{algorithm}

\section{Theoretical Proofs}

Our analysis builds on standard high-probability concentration bounds for Gaussian process (GP) regression
\cite{srinivas2010gaussian,kandasamy2016gaussian,av2022human}, which hold uniformly over arbitrary evaluation sequences and are therefore independent of the acquisition strategy. We combine these bounds with Lipschitz properties of the Chebyshev utility to translate objective-level approximation errors into errors in the inferred mixture utility. Since our acquisition jointly depends on GP surrogate accuracy and latent preference estimation, we derive a high-probability simple regret decomposition that separates the contributions of GP approximation, preference archetype estimation, mixture-weight estimation, and acquisition optimization.

\begin{theorem}[Simple regret decomposition]
\label{thm:regret_safe}

Let $\mathcal X$ be compact and let 
$\mathbf f = (f_1,\dots,f_L):\mathcal X\to\mathbb R^L$ 
be an $L$-objective function. Assume:

\begin{itemize}
\item[(A1)] Each objective $f_\ell$ lies in an RKHS $\mathcal H_\ell$ with kernel $k_\ell$ and 
$\|f_\ell\|_{\mathcal H_\ell}\le B_f$.

\item[(A2)] For all $x\in\mathcal X$, 
$\|\mathbf f(x)\|_\infty \le B_y$.

\item[(A3)] Preference weights satisfy 
$w_{k\ell} \ge c_w > 0$ for all modes $k$ and objectives $\ell$.

\item[(A4)] Observations are corrupted by independent $\sigma^2$–sub-Gaussian noise.
\end{itemize}

Define the Chebyshev utility for minimization
\[
U(\mathbf y; \mathbf w)
=
- \min_{1\le \ell \le L}
\frac{y_\ell}{w_\ell},
\]
and the mixture utility
\[
U_{\mathrm{mix}}(\mathbf y; \eta, \{w_k\})
=
\sum_{k=1}^K \eta_k U(\mathbf y; w_k).
\]

Let $(\eta^\star, w_k^\star)$ denote the true mixture parameters and define
\[
U^\star
=
\sup_{x\in\mathcal X}
U_{\mathrm{mix}}(\mathbf f(x); \eta^\star, \{w_k^\star\}).
\]

After $T$ evaluations $x_1,\dots,x_T$, define the simple regret
\[
R_T
=
U^\star
-
\max_{t\le T}
U_{\mathrm{mix}}(\mathbf f(x_t); \eta^\star, \{w_k^\star\}).
\]

Then, with probability at least $1-\delta$,

\[
R_T
\;\le\;
C_1
\sqrt{\frac{\beta_T \gamma_T}{T}}
+
C_2 \, \Delta_w(T)
+
C_3 \, \Delta_\eta(T)
+
\varepsilon_T,
\]

where $\gamma_T$ is the maximal GP information gain, $\beta_T$ is the standard GP confidence parameter, $\Delta_w(T)=\max_k \|\hat w_k - w_k^\star\|_1$, $\Delta_\eta(T)=\|\hat \eta - \eta^\star\|_1$, $\varepsilon_T$ is the optimization error in maximizing the acquisition, $C_1,C_2,C_3$ depend only on $B_y$ and $c_w$.
\end{theorem}

\begin{remark}
Theorem~\ref{thm:regret_safe} provides an error decomposition for the final simple regret rather than a classical sublinear cumulative regret guarantee. Its primary purpose is to identify the distinct sources of error introduced by objective approximation, latent preference estimation, and numerical optimization.
\end{remark}

Under consistent preference estimation, i.e., when $\Delta_w(T)\rightarrow0$, $\Delta_\eta(T)\rightarrow0$, and the acquisition optimization error $\varepsilon_T\rightarrow0$, the bound is dominated by the GP approximation term, which decreases at the standard GP convergence rate $\mathcal{O}\!\left(\sqrt{\beta_T\gamma_T/T}\right)$. The detailed proof is provided in the supplementary.

\section{Related Work}
While there are several recent efforts along the direction of multi-objective Bayesian optimization \cite{li2025expensive, haddadnia2025botier, hung2025boformer}, a line of works integrate human-in-the-loop for Bayesian optimization \cite{xu2024principled, av2022human}. Towards interactive multi-objective optimization, \cite{ozaki2024multi} proposes a framework that optimizes multiple objectives by actively querying the decision-maker in order to steer Bayesian optimization towards user's preferred trade-offs. There are limited efforts in the direction of many-objectove optimization. Classical many-objective Bayesian optimization increasingly avoids recovering large Pareto sets and instead targets a principled single compromise. The Kalai-Smorodinsky (KS) solution equalizes benefit ratios from a disagreement point to utopia; subsequent work further introduces a copula-invariant variant (CKS) to remove sensitivity to monotone rescalings, together with a GP-based stepwise-uncertainty-reduction (SUR) scheme that scales to 4-9 objectives \cite{binois2020kalai}. These methods elegantly collapse decision making to a point solution, but they assume access to full objective vectors (rather than comparisons) and a single underlying taste of utility. A complementary line seeks a small set of solutions that collaboratively “cover” many objectives. Recent work proposes Tchebycheff-set (TCH-Set) scalarization and a smooth variant (STCH-Set) that mitigates the non-smooth max operator, with theory and empirical studies showing that a handful of solutions can handle tens of objectives effectively \cite{lin2024few}. Such approaches reduce downstream decision complexity, yet still optimize from full objective feedback and typically posit a single, fixed scalarization rather than heterogeneous preferences. Orthogonal to both is objective reduction where detecting redundant objectives via similarity between GP predictive distributions and dropping them to save evaluations. This can preserve solution quality on toy, synthetic, and real setups while cutting cost \cite{martin2021many}, but it maintains a single user model and does not reason about preference heterogeneity. Relative to single–compromise and few–solution paradigms \cite{binois2020kalai,lin2024few, jiang2026we}, our framework operates directly on preferences and embraces heterogeneity as signal, not noise. Relative to objective–reduction \cite{martin2021many}, we keep all objectives but learn which trade–off modes matter and when, delivering interpretable decisions under many objectives.

\begin{figure*}
\centering
\begin{subfigure}{0.30\textwidth}
  \includegraphics[width=\linewidth]{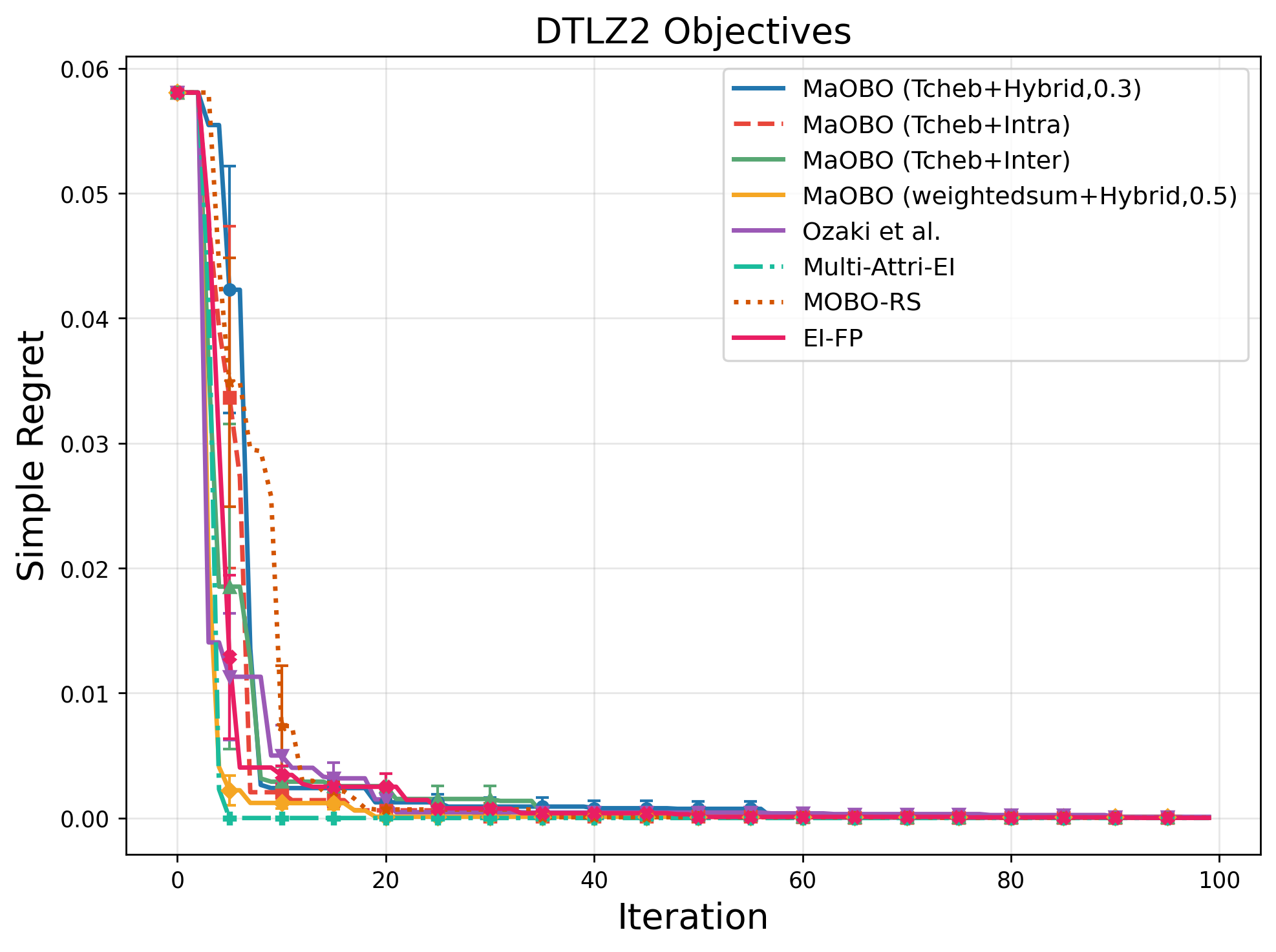}
  \caption{DTLZ2}
\end{subfigure}\hfill
\begin{subfigure}{0.30\textwidth}
  \includegraphics[width=\linewidth]{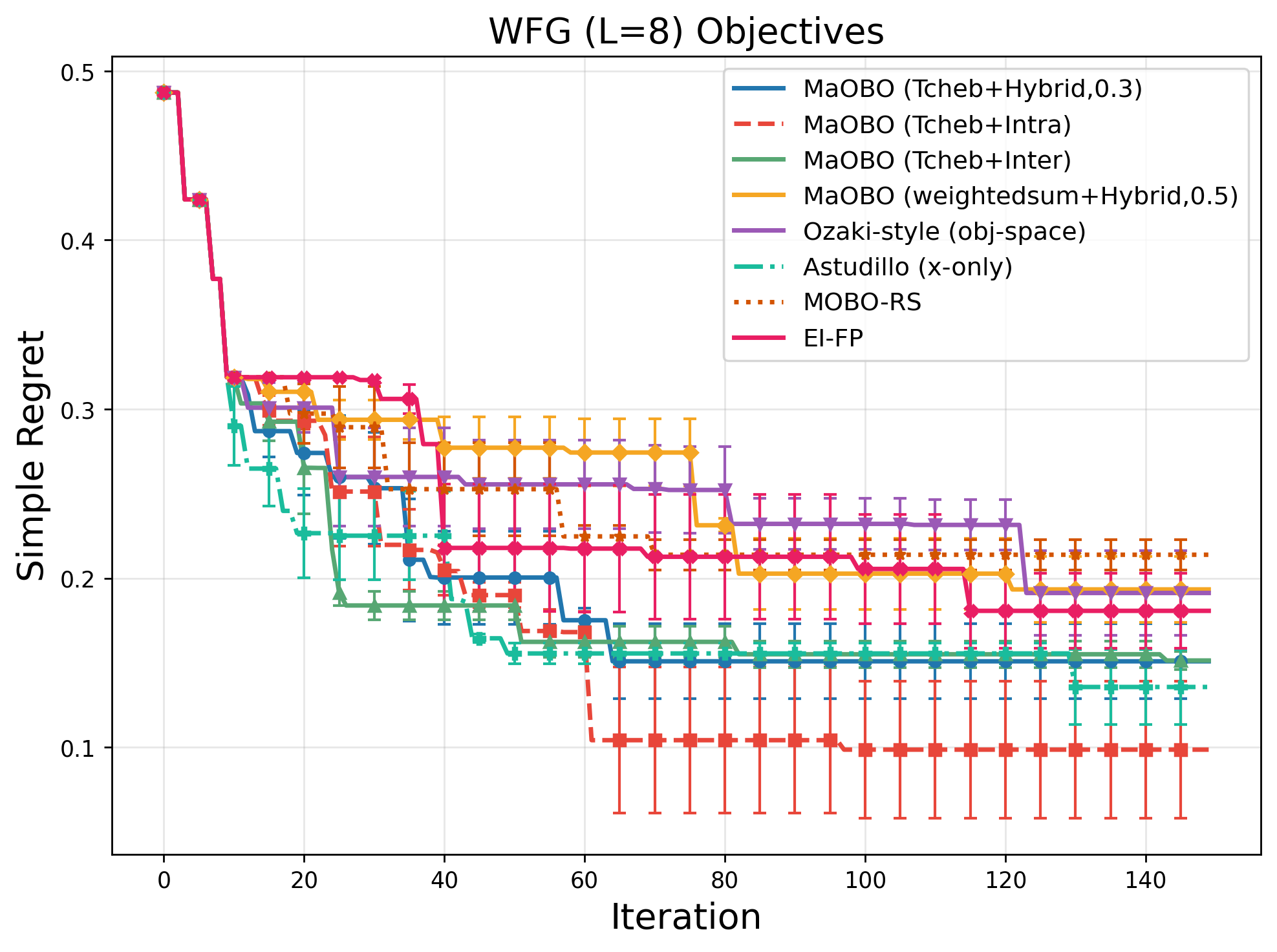}
  \caption{WFG}
\end{subfigure}\hfill
\begin{subfigure}{0.30\textwidth}
  \includegraphics[width=\linewidth]{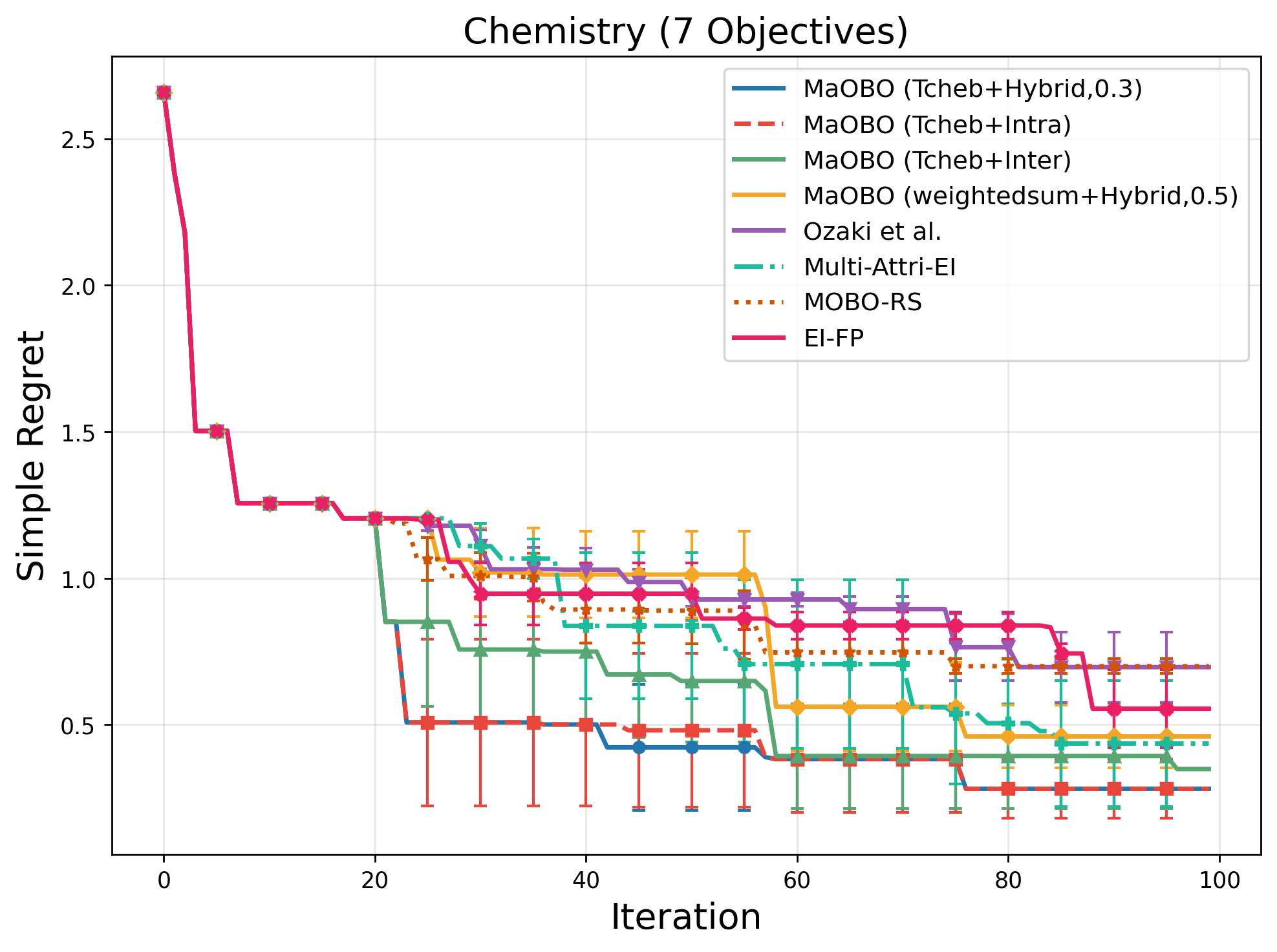}
  \caption{Chemistry}
\end{subfigure}

\vspace{-0.5mm}

\begin{subfigure}{0.23\textwidth}
  \includegraphics[width=\linewidth]{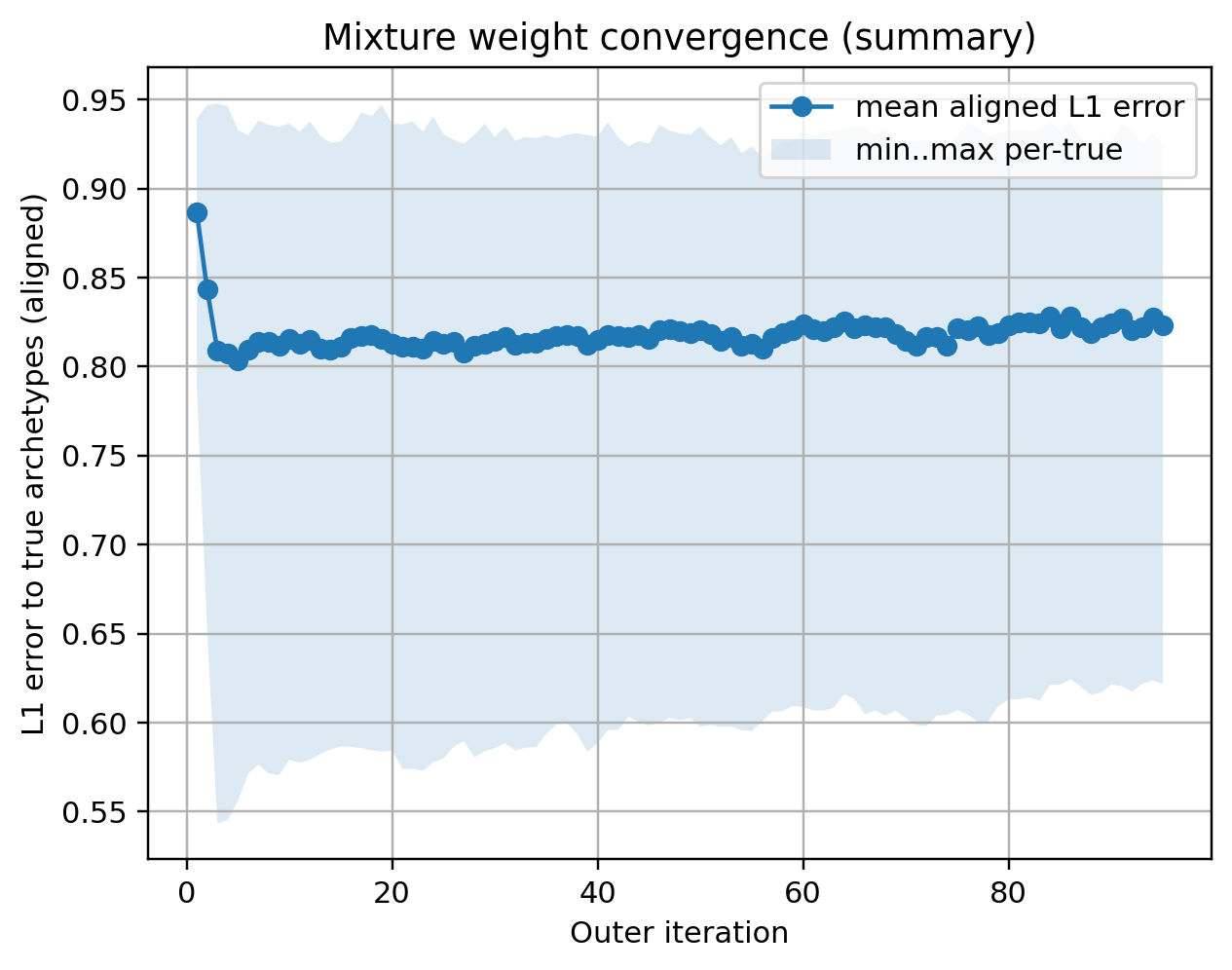}
  \caption{Clusterless}
\end{subfigure}\hfill
\begin{subfigure}{0.23\textwidth}
  \includegraphics[width=\linewidth]{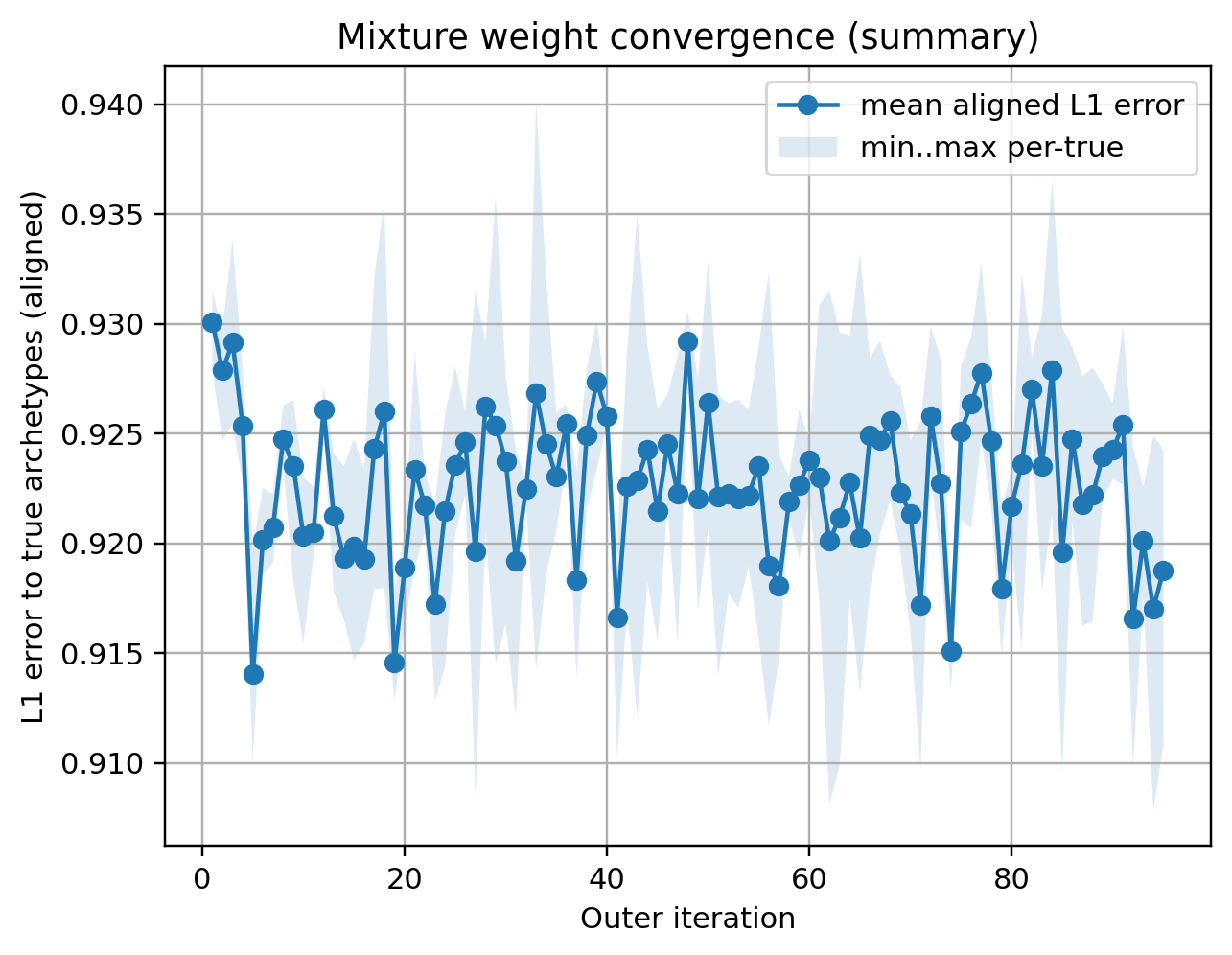}
  \caption{Intra}
\end{subfigure}\hfill
\begin{subfigure}{0.23\textwidth}
  \includegraphics[width=\linewidth]{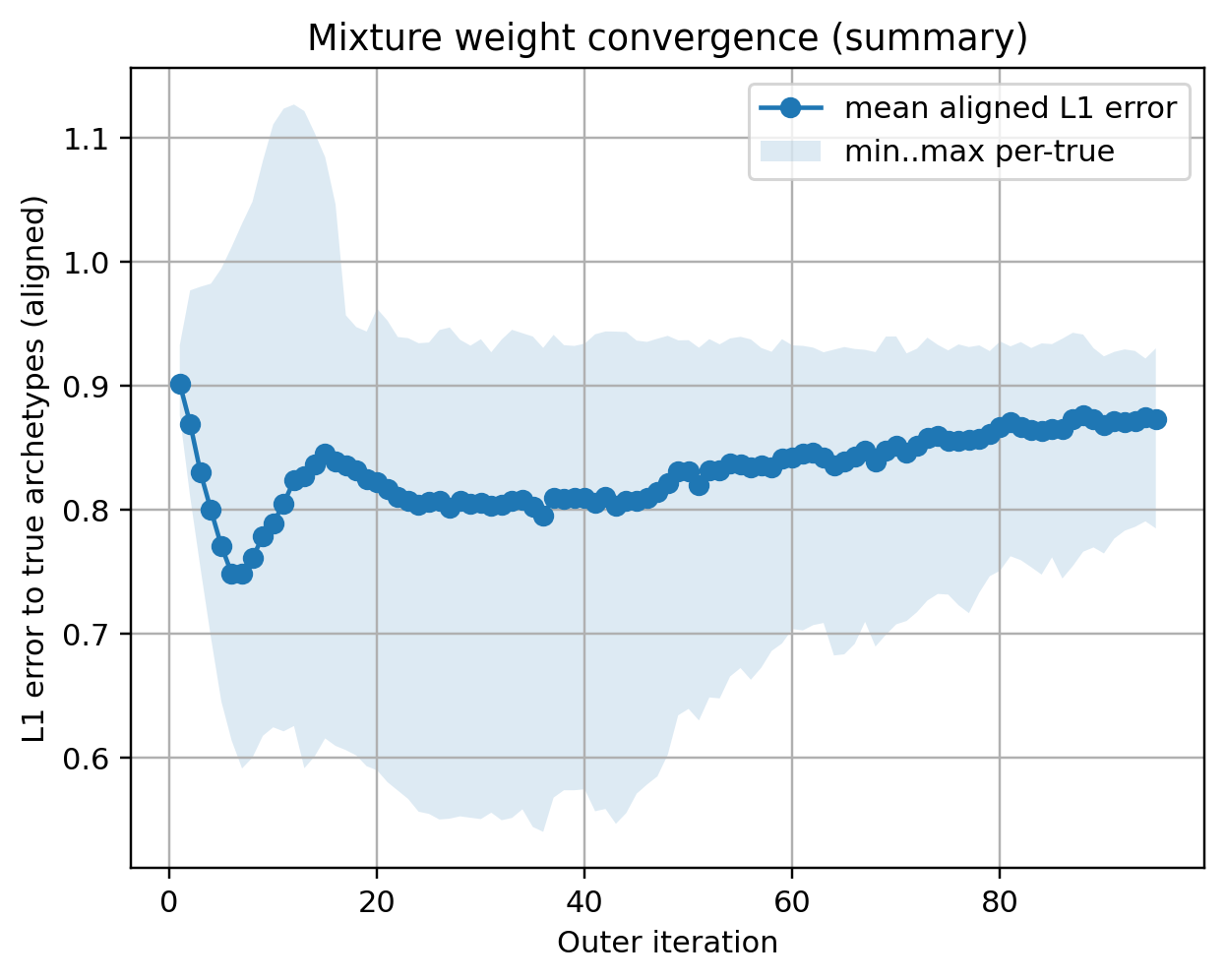}
  \caption{Inter}
\end{subfigure}\hfill
\begin{subfigure}{0.23\textwidth}
  \includegraphics[width=\linewidth]{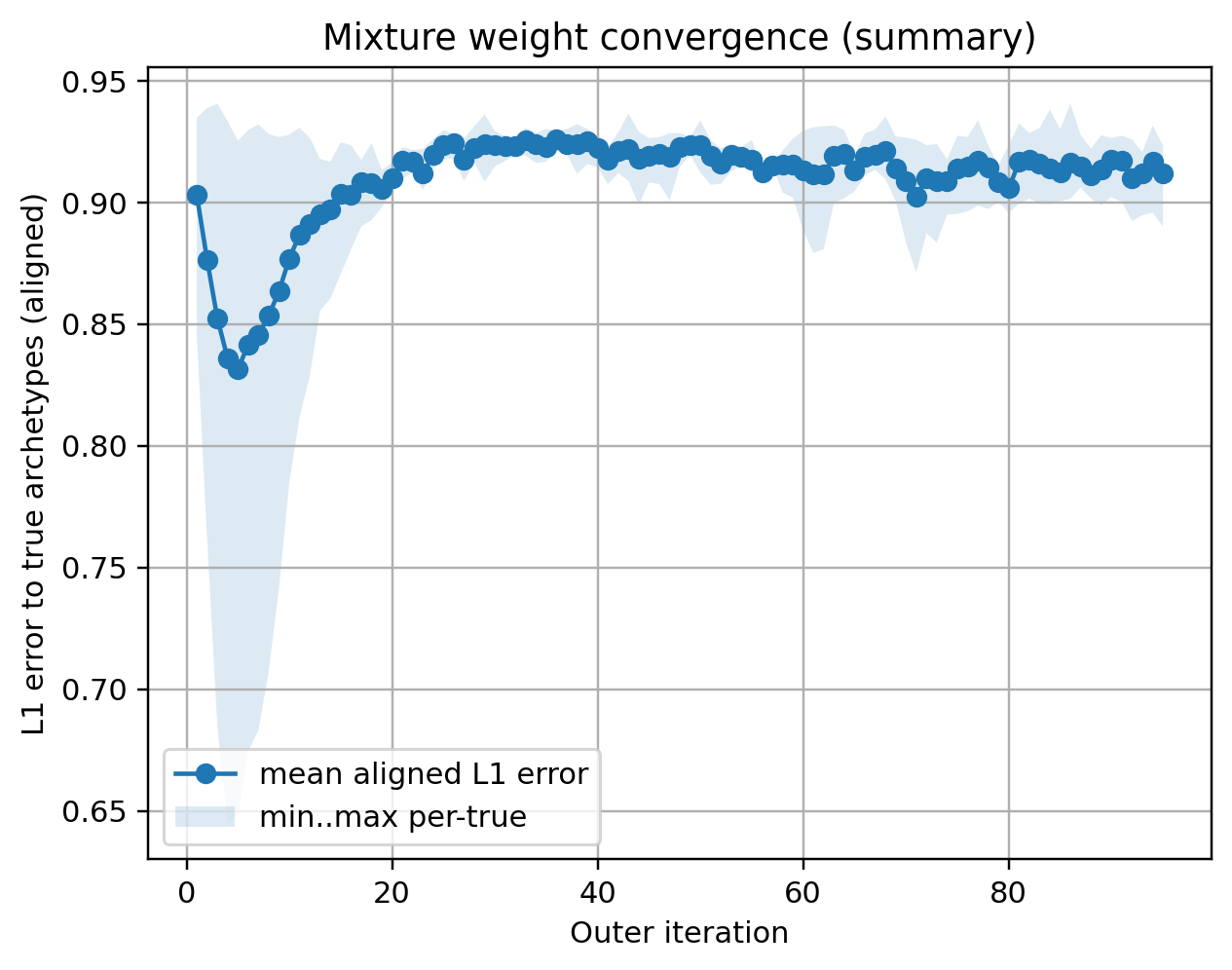}
  \caption{Hybrid}
\end{subfigure}
\vspace{-2mm}
\caption{ (Top) Mixture-aware policies reduce regret faster and achieve lower final regret than unimodal and random-scalarization baselines, with Hybrid performing best overall. This is visible from the steeper early decline and lower terminal curves of Inter and Hybrid across DTLZ2, WFG, and chemical task, while Clusterless and scalarization methods plateau higher with greater variability. Curves show mean simple regret over five independent runs; vertical bars denote $\pm$ one standard error.(Bottom)Hybrid most accurately and stably recovers true archetypes, showing the largest early drop and lowest final error with the narrowest band; Inter drops quickly but plateaus, Intra refines slowly with higher variance, and Clusterless remains flat. We report mean aligned L1 error per outer iteration (band = per-true min–max; lower is better). Dataset: DTLZ2, persistent context.}
\label{fig:baseline_comparison}
\end{figure*}

\section{Experiments}
\paragraph{Datasets} We evaluate on two standard many-objective optimization benchmarks, DTLZ2~\cite{deb2005scalable} and WFG9~\cite{huband2005scalable}, representing smooth and highly non-separable optimization landscapes, respectively. To simulate heterogeneous preferences, we generate pairwise comparisons from $K$ latent preference archetypes, each defined by a Chebyshev utility with weight vector $\mathbf w_k$. For DTLZ2 ($L=6$) and WFG9 ($L=8$), archetypes emphasize different objective groups by assigning $80\%$ of the weight to one group and distributing the remainder across the others. We consider both an i.i.d. setting, where latent modes are sampled independently from $\mathrm{Categorical}(\eta^\star)$, and a persistent setting, where mode identity evolves according to a sticky Markov process with persistence probability $\rho$. We further evaluate our method on a real-world chemical process design case study for polyethylene terephthalate (PET) production~\cite{wang2022machine}, comprising 10,000 simulated designs with 12 decision variables and seven objectives (return on investment and six life-cycle assessment metrics). We construct three stakeholder archetypes representing economic, environmental, and health priorities, with true mixture weights $\eta=[0.40,0.35,0.25]$ and persistence $\rho=0.8$. Additional dataset and simulation details are provided in the Supplementary.

\paragraph{Learning and Policies}
We adopt Chebyshev scalarization for its ability to represent non-convex trade-offs and its standard use in many-objective optimization, though the framework can accommodate alternative scalarizations. Each objective is modeled with an independent GP, and preference modes are learned via a truncated stick-breaking mixture using SVI with priors $v_k \sim \mathrm{Beta}(1,\alpha)$ and $\mathbf w_k \sim \mathrm{Dirichlet}(\boldsymbol\beta)$. Although we fix a truncation level $K$, the stick-breaking prior shrinks redundant components toward negligible mass, so effective modes correspond to components with non-trivial posterior weight. The hybrid parameter $\lambda$ is treated as a hyperparameter and selected via validation.

\paragraph{Baselines}
We compare against following optimization baselines: (i)\textit{Ozaki}~\cite{ozaki2023multi}, learns preference vector interactively using MCMC with single utility function; (ii) \textit{Multi-Attr-EI}~\cite{astudillo2020multi}, which infers a single preference vector using MAP estimation; (iii) \textit{MOBO-RS}~\cite{paria2020flexible}, based on random scalarizations; (iv) \textit{EI-FP}, expected improvement with known fixed preferences; acting as upper performance bound and (v) \textit{MaOBO-WS}, our framework using weighted-sum scalarization instead of Chebyshev. For preference query acquisition, we compare \textit{Random}, \textit{Clusterless} (unimodal BALD using the mixture-mean preference vector), \textit{Inter} (mode identification; Eq.~\ref{eq:inter_query}), \textit{Intra} (within-mode refinement; Eq.~\ref{eq:intra_query}), and \textit{Hybrid} (convex combination of Inter and Intra; Eq.~\ref{eq:hybrid_query}). Existing preferential BO methods assume a single latent utility, while random scalarization explores diverse trade-offs without inferring heterogeneous preferences. In contrast, our method explicitly models multi-modal, dynamically switching preference archetypes.


\paragraph{Metrics}
We report (i) \textit{Simple Regret},
$
r_t = U^* - \max_{i\le t} U(f(x_i),w_{\mathrm{true}})
$ (lower is better); (ii) \textit{Archetype Recovery}, measured as Hungarian-aligned mean $\ell_1$ error between inferred and true archetypes; (iii) \textit{Mixture-weight trajectories}, tracking convergence of posterior means $\hat{\eta}$ to the true mixture weights; and (iv) \textit{Per-component $\ell_1$ error}, which reveals mode collapse, label switching, and convergence of individual archetypes.

\section{Results and Analysis}

Figure~\ref{fig:baseline_comparison} (Top) compares simple regret under the true mixture utility across outer iterations, with uncertainty estimate over five seeds. Across benchmarks, the mixture-aware methods achieve faster regret reduction and lower final regret than unimodal and random-scalarization baselines. For instance, \cite{ozaki2023multi} achieves regret as $0.697$ as opposed to $0.282$ with $60\%$ reduction, reflecting that prior methods assume a single utility, whereas proposed method models heterogeneous preferences via latent archetypes. On DTLZ2, all methods improve rapidly due to the smooth landscape, but Inter and Hybrid converge more consistently to near-zero regret with lower variance, while Clusterless and fixed-preference baselines lag slightly, indicating consistent gains from explicit mode identification even in well-behaved settings. On the more irregular WFG benchmark, the separation is clearer: Hybrid exhibits the steepest early decline and the lowest final regret, whereas Clusterless and random scalarization plateau at higher levels, suggesting that preference averaging fails to resolve complex trade-offs. In the chemical process design task, the gap is most pronounced: Hybrid converges faster and stabilizes at lower regret, while unimodal and scalarization baselines converge more slowly and with higher variance, highlighting the benefit of modeling heterogeneous archetypes in structured, multimodal trade-off landscapes. Figure~\ref{fig:baseline_comparison} (Bottom) shows archetype recovery for a persistent user on DTLZ2. Clusterless exhibits a flat curve with a wide band, indicating mode averaging and poor archetype separation. Intra-queries reduce error within active modes but produce unstable overall alignment, reflected in persistent variability. Inter-queries yield a sharp early drop by rapidly disambiguating archetypes, followed by slower refinement. Hybrid combines both effects: rapid initial mode identification followed by steady within-mode calibration, achieving lowest final aligned error and the narrowest uncertainty band.

Figure~\ref{fig:cluster_mixture_weights} illustrates the evolution of inferred mixture weights under persistent preferences with true weights $(0.40,0.35,0.25)$. Hybrid rapidly reallocates probability mass during the early iterations to identify the latent preference modes before converging to stable estimates close to the ground truth. The smallest archetype remains slightly underestimated, reflecting the increased difficulty of learning minority preference modes from limited preference queries. Importantly, no spurious dominant component emerges and the mixture weights stabilize after approximately 20 iterations. Simple regret alone does not reveal how different preference archetypes are explored. We therefore visualize evaluated solutions relative to a reference Pareto front. As shown in Fig.~\ref{fig:cluster_mixture_weights}, Hybrid explores multiple regions of the Pareto frontier before concentrating on the mixture-preferred region. Compared with competing methods, it achieves better localization of solutions near the preferred trade-off surface, particularly along the ROI--Toxicity projection.


\begin{figure}
\centering
\begin{subfigure}{0.23\textwidth}
  \includegraphics[width=\linewidth]{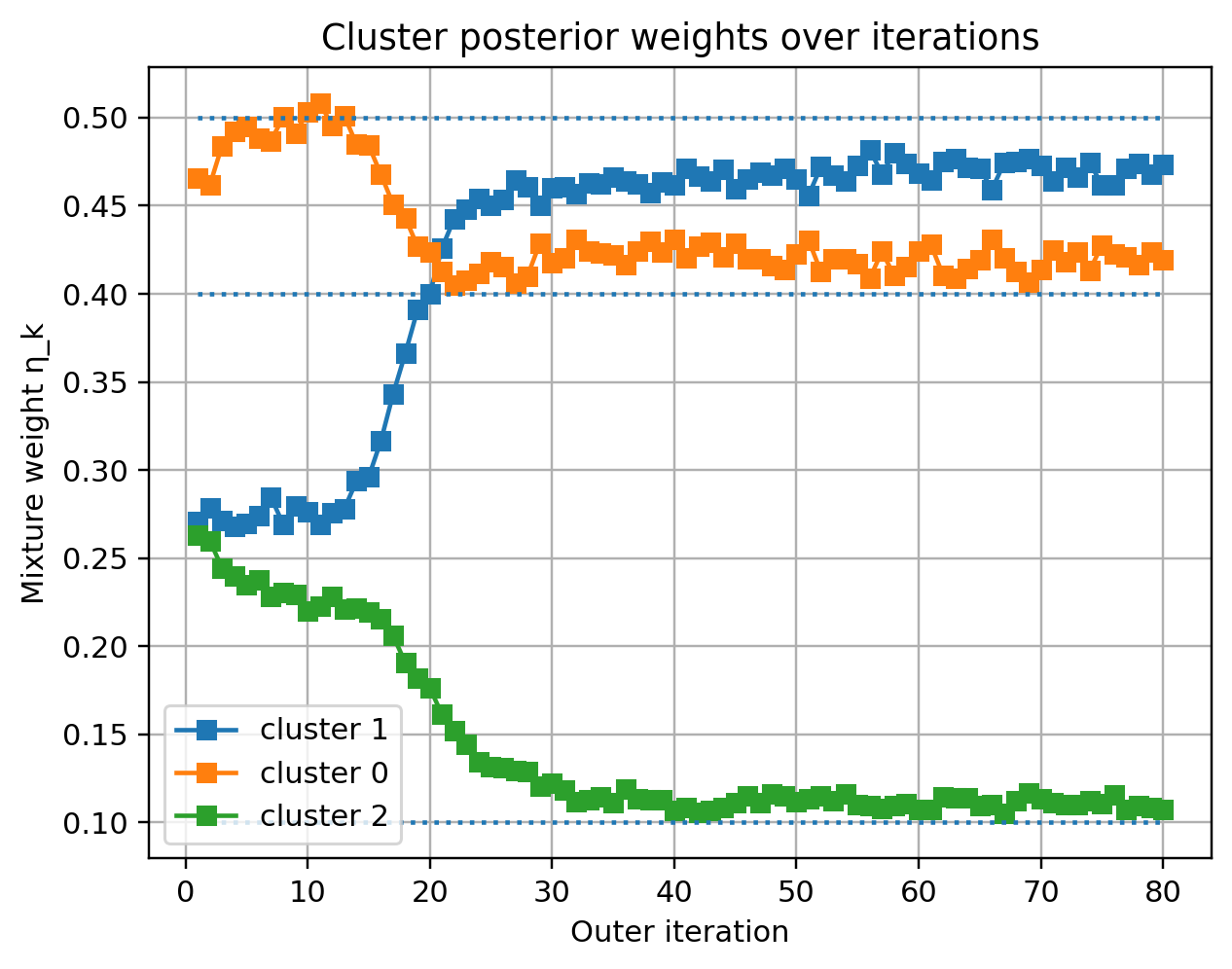}
  \caption{Hybrid: correct mode}
\end{subfigure}
\begin{subfigure}{0.23\textwidth}
  \includegraphics[width=\linewidth]{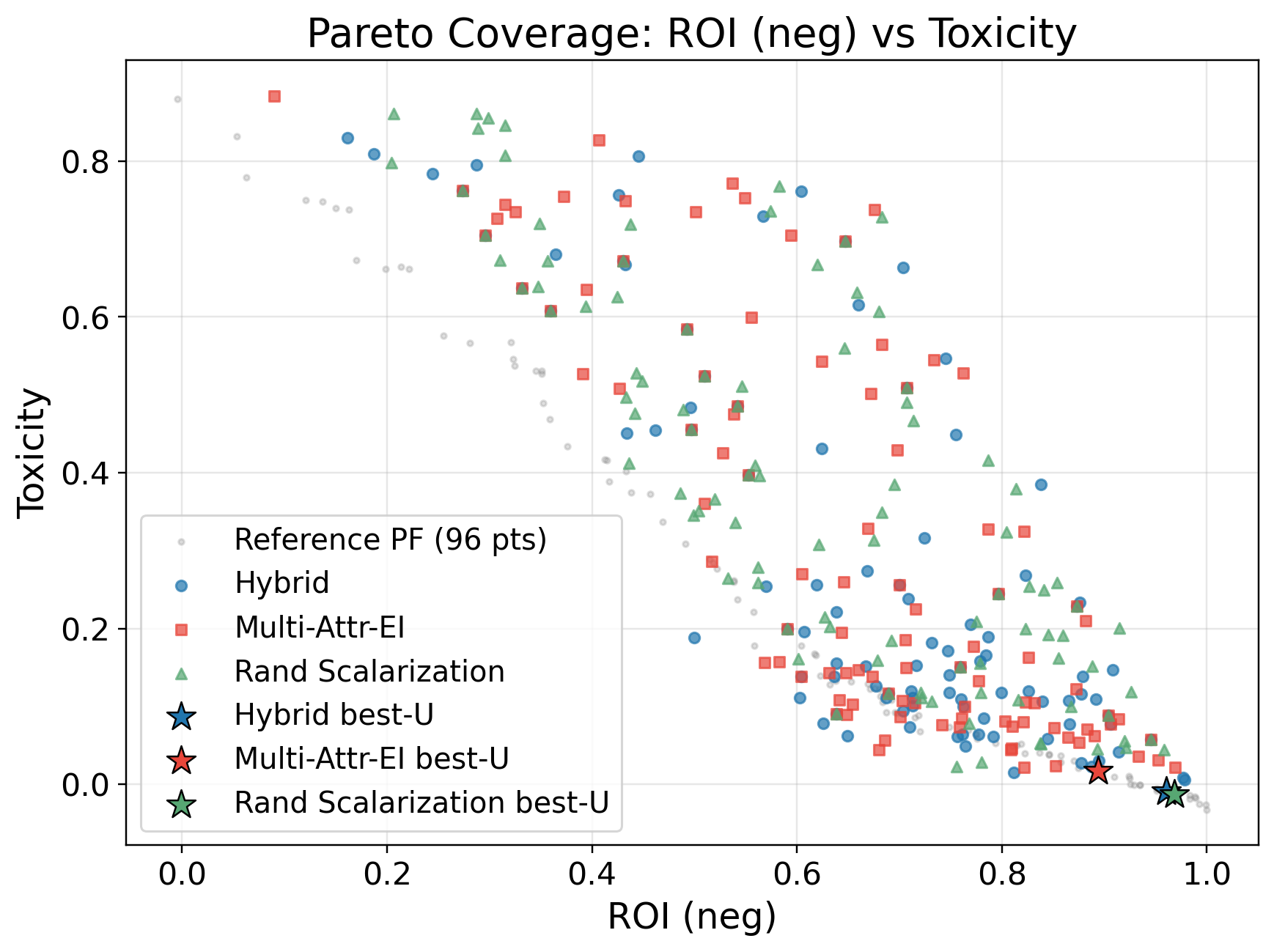}
  \caption{ROI vs Toxicity}
\end{subfigure}
\caption{Mixture weight trajectory for chemical process design process. (left) Hybrid quickly identifies and stabilizes near the correct mode proportions. (right) Grey: reference Pareto front; coloured: evaluated solutions; stars: final best-utility solution. Hybrid converges to the preferred Pareto region.}
\label{fig:cluster_mixture_weights}
\vspace{-5mm}
\end{figure}

\begin{figure}
\centering
\begin{subfigure}{0.23\textwidth}
  \includegraphics[width=\linewidth]{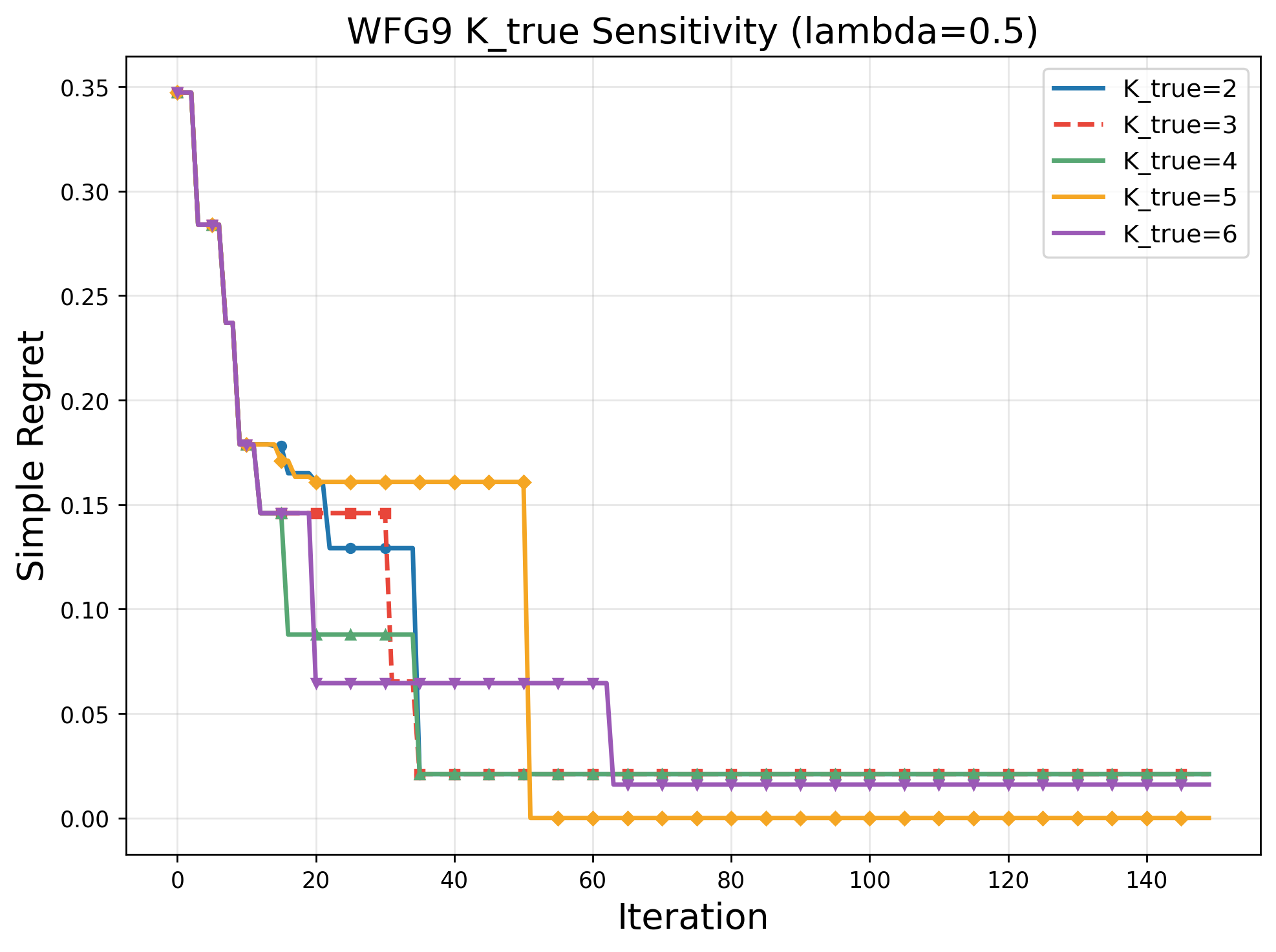}
  \caption{Scalability with increasing preference heterogeneity}
\end{subfigure} 
\begin{subfigure}{0.23\textwidth}
  \includegraphics[width=\linewidth]{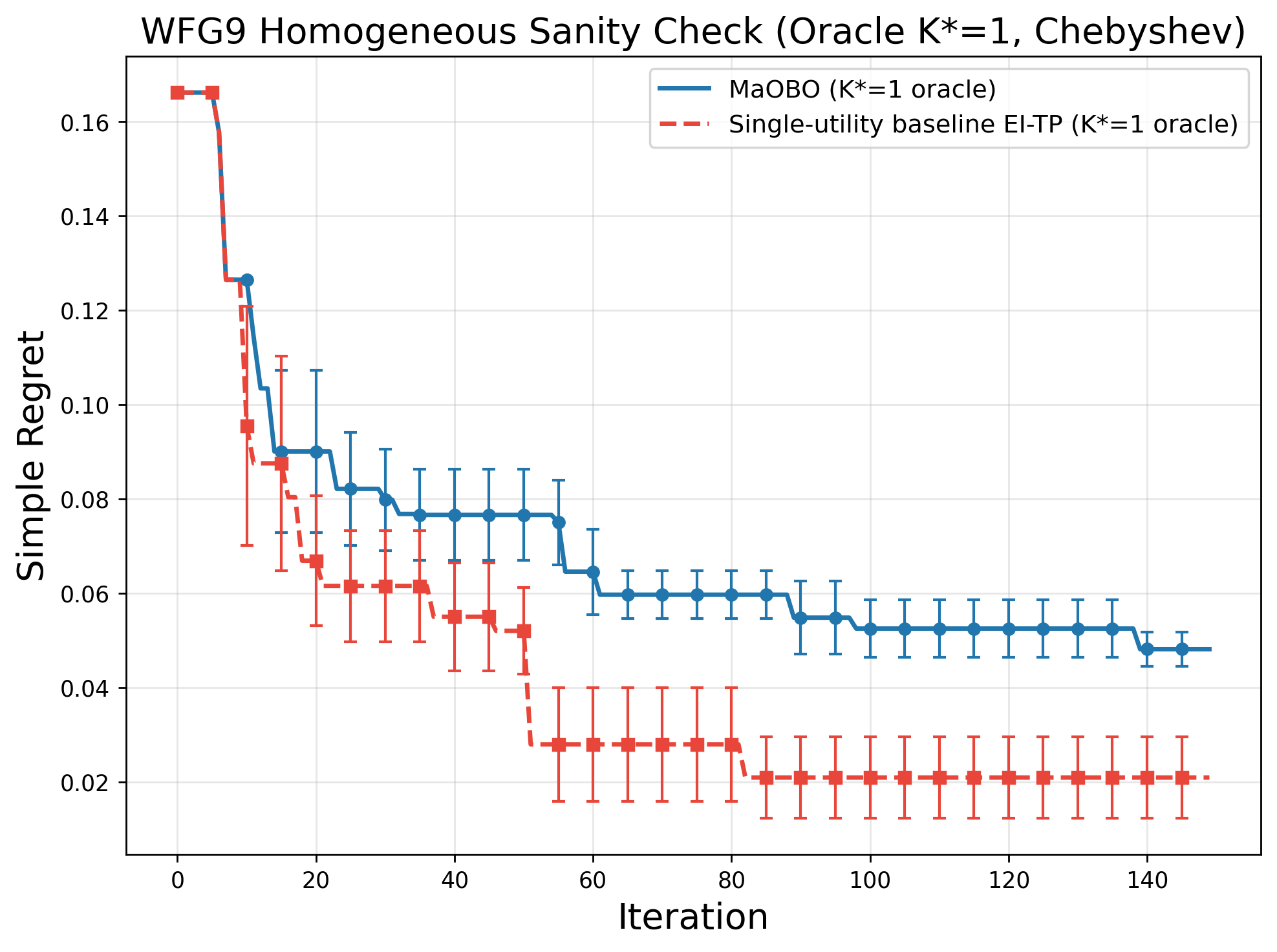}
  \caption{Homogenous Oracle ($K^*=1$)}
\end{subfigure}
\caption{Robustness Analysis for WFG}
\vspace{-3mm}
\label{fig:robustness_analysis}
\vspace{-2mm}
\end{figure}

\paragraph{Robustness and Scalability}

We evaluate the method under varying preference complexity and model structure. Figure~\ref{fig:robustness_analysis}(a) varies the true number of preference archetypes from \(K^\star=2\) to \(K^\star=6\). Increasing \(K^\star\) generally slows convergence because more preference observations are required to distinguish the latent archetypes. Nevertheless, all settings attain low final regret, indicating that the method scales gracefully with increasing preference heterogeneity. Figure~\ref{fig:robustness_analysis}(b) considers a homogeneous oracle with \(K^\star=1\). As expected, a specialized single-utility baseline performs best in this setting, while the proposed mixture model remains competitive by concentrating posterior mass on a single dominant archetype. Additional studies on \(K\)-misspecification, utility misspecification, and sensitivity to the hybrid parameter \(\lambda\) are reported in the Supplementary.

\section{Conclusion}

We introduced a mixture-aware active preference learning framework for many-objective Bayesian optimization that moves beyond the standard assumption of a single latent utility function. By representing decision-maker preferences as latent preference archetypes, the proposed framework jointly infers heterogeneous preference patterns from pairwise comparisons and actively acquires informative preference queries. This representation naturally decomposes preference elicitation into archetype identification and within-archetype refinement, leading to inter-archetype, intra-archetype, and hybrid information-theoretic query strategies that improve preference-guided optimization. Experiments on synthetic benchmarks and a real-world chemical process design case study demonstrate that the proposed framework more efficiently identifies preferred regions of the Pareto front while consistently improving optimization performance over existing preference-based Bayesian optimization methods. Beyond optimization, the learned preference archetypes recover interpretable latent preference archetypes, and the proposed diagnostics quantify archetype recovery and preference calibration beyond conventional regret-based evaluation. 

Together, these results establish latent preference archetypes as a principled foundation for active preference learning in many-objective Bayesian optimization. The proposed framework is orthogonal to existing many-objective Bayesian optimization methods and can be combined with alternative surrogate models, scalarizations, or acquisition functions. While our evaluation uses simulated preferences and a real-world engineering case study, an important next step is validation with human decision makers, broader real-world applications, and adaptive models of non-stationary preferences.

\bibliography{aaai2027}

\clearpage
\appendix

\begin{center}
{\Huge\bfseries Supplementary Material}

\vspace{0.2cm}

\end{center}

\section{Detailed Experimental Details }

\subsection{Datasets}

We evaluate on two standard many-objective benchmarks, DTLZ2~\cite{deb2005scalable} and WFG9~\cite{huband2005scalable}, together with a real-world chemical process design benchmark for polyethylene terephthalate (PET) production~\cite{wang2022machine}.

\paragraph{DTLZ2:}
DTLZ2 uses $L=6$ objectives and $d=7$ decision variables over $[0,1]^7$. All objectives are minimized.

\paragraph{WFG9:}
WFG9 uses $L=8$ objectives and $d=34$ variables ($k=14$, $l=20$). Variables are normalized to $[0,1]$. The benchmark exhibits bias and mixed separability, making it challenging for preference learning.

\paragraph{PET process design:}
The PET benchmark contains 10,000 simulated process designs with 12 normalized decision variables and seven objectives consisting of ROI and six Life Cycle Assessment (LCA) indicators (GWP, terrestrial acidification, water consumption, fossil depletion, surplus ore and human toxicity). We treat the dataset as a fixed black-box many-objective optimization problem.

\subsection{Simulated Datasets}
We use well-known MOO benchmarks for our experimental setup, called DTLZ \cite{deb2005scalable} and WFG suite \cite{huband2005scalable}. We evaluate on DTLZ2 with \(L=6\) objectives and \(d=7\) decision variables
(minimization), search space \([0,1]^d\). At each design \(\mathbf x\in[0,1]^7\) we observe
\(\mathbf y=\mathbf f(\mathbf x)\). We also use WFG9 with $L=8$ objectives and $d=34$ decision variables. Following common practice, we set the WFG position parameter $k=2(L-1)=14$ and distance parameter $l=20$ so $d=k+l$. Variables are scaled to $[0,1]$ and all objectives are minimized. A design
$\mathbf x\!\in\![0,1]^d$ evaluates to an objective vector
$\mathbf y=\mathbf f(\mathbf x)\in\mathbb{R}^8$. WFG9 induces challenging biases and mixed separability, making it a stronger testbed for query policies.

To construct distinct but overlapping preference archetypes, we partition the objectives into disjoint groups and assign structured weight vectors. For the DTLZ suite (\(L=6\)), we define three groups
$
\mathcal G_1=\{0,1\}, \quad 
\mathcal G_2=\{2,3\}, \quad 
\mathcal G_3=\{4,5\}$ while for the WFG suite (\(L=8\)) we define four groups
$
\mathcal G_1=\{0,1\}, \;
\mathcal G_2=\{2,3\}, \;
\mathcal G_3=\{4,5\}, \;
\mathcal G_4=\{6,7\}.
$
Each mode allocates \(80\%\) of its total mass uniformly across its dominant group \(\mathcal G_k\) and distributes the remaining \(20\%\) uniformly over the other objectives. The resulting vector is renormalized to ensure \(\mathbf w_k \in \Delta^{L-1}\). This construction yields clearly separated trade-off profiles while maintaining nonzero sensitivity to all objectives.

Let \(Z_t \in \{1,\dots,K\}\) denote the active preference mode at query \(t\). Under an iid context regime, each query is generated independently according to a fixed mixture \(Z_t \sim \mathrm{Categorical}(\eta^\star)\), with \(\eta^\star=(0.5,0.3,0.2)\) for DTLZ and \(\eta^\star=\{(0.60,0.25,0.15,0.10)\), \(0.4, 0.3, 0.3, 0.1\)\} for WFG. Under a persistent context regime, the mode evolves according to a stay-or-resample process with persistence parameter \(\rho \in [0,1)\). The initial mode is drawn as \(Z_1 \sim \mathrm{Categorical}(\eta^\star)\), and at each subsequent query the mode remains unchanged with probability \(\rho\) and is resampled from \(\mathrm{Categorical}(\eta^\star)\) with probability \(1-\rho\). The expected run length within a single mode is \(1/(1-\rho)\), and the expected number of switches over \(T\) queries is \((1-\rho)(T-1)\).

\subsection{Real-World: PET Process Design}

We evaluate the proposed framework on a real-world multi-objective process design problem for polyethylene terephthalate (PET) production. The dataset comprises 10{,}000 previously evaluated process designs. Each design is represented by a 12-dimensional continuous decision vector $\mathbf x \in \mathbb{R}^{12}$ corresponding to operating and design parameters such as temperatures, pressures, residence times, and feed ratios. Each variable is bounded within physically meaningful limits and normalized to $[0,1]$ for modeling purposes. For each design, we consider $L=7$ objectives consisting of Return on Investment (ROI), life-Cycle Assessment (LCA) indicators, including global warming potential (GWP), terrestrial acidification, water consumption, fossil depletion, surplus ore, and human toxicity. For the PET benchmark we construct three stakeholder archetypes representing economic, environmental and health-oriented priorities. Mixture weights are fixed to
$\eta=(0.40,0.35,0.25)$
with persistence
$\rho=0.8$.

\subsection{Interaction Contexts}

We consider two preference-generation regimes.

\paragraph{I.i.d:}
The active archetype is sampled independently at each comparison,
$
Z_t\sim\mathrm{Categorical}(\eta^\star).
$

\paragraph{Persistent:}
Under the persistent regime, the active archetype remains unchanged with probability $\rho$ and is resampled from $\mathrm{Categorical}(\eta^\star)$ with probability $1-\rho$, yielding temporally consistent preferences with occasional switches. This models temporally consistent preferences with occasional switches.

\subsection{Implementation Details}

Experiments are implemented in Python using JAX, NumPyro, GPy and Optuna. Each objective is modeled by an independent RBF Gaussian process. Preference inference uses a truncated stick-breaking Dirichlet-process mixture optimized with stochastic variational inference (Adam, learning rate $10^{-3}$, 500 updates). Preference queries are selected using the proposed clusterless, inter-, intra-, or hybrid criteria, while new evaluations are chosen via Monte Carlo expected improvement (12 MC samples, 10 L-BFGS-B restarts). All methods share identical initial designs, random seeds, and evaluation budgets. Results are reported as mean$\pm$std over three independent runs. All experiments were conducted on a MacBook Pro equipped with an Apple M2 Pro processor (10-core CPU, 16-core GPU), 32\,GB unified memory, running macOS. The implementation used Python 3.12 with JAX, NumPyro, GPy, and Optuna. No distributed computing or external GPU cluster was used.
\subsection{Additional Metrics}

Besides simple regret, we report 1) Mixture-weight calibration by comparing inferred $\hat{\eta}$ with ground-truth $\eta^\star$ after Hungarian alignment. 2) Preference estimation error using mixture-mean, MAP-cluster and expected $L_1$ errors.

\begin{table}[t]
\centering
\small
\begin{tabular}{ll}
\toprule
Symbol & Meaning \\
\midrule
$L$ & number of objectives \\
$K$ & truncation level (max modes) \\
$\ell$ & objective index \\
$k$ & preference mode index \\
$\mathbf w_k$ & weight vector for mode $k$ \\
$\eta_k$ & mixture weight of mode $k$ \\
$r_i$ & latent mode for comparison $i$ \\
$\sigma_u$ & preference noise parameter \\
\bottomrule
\end{tabular}
\caption{Notation summary.}
\end{table}

\begin{figure}
\centering
\begin{subfigure}{0.23\textwidth}
  \includegraphics[width=\linewidth]{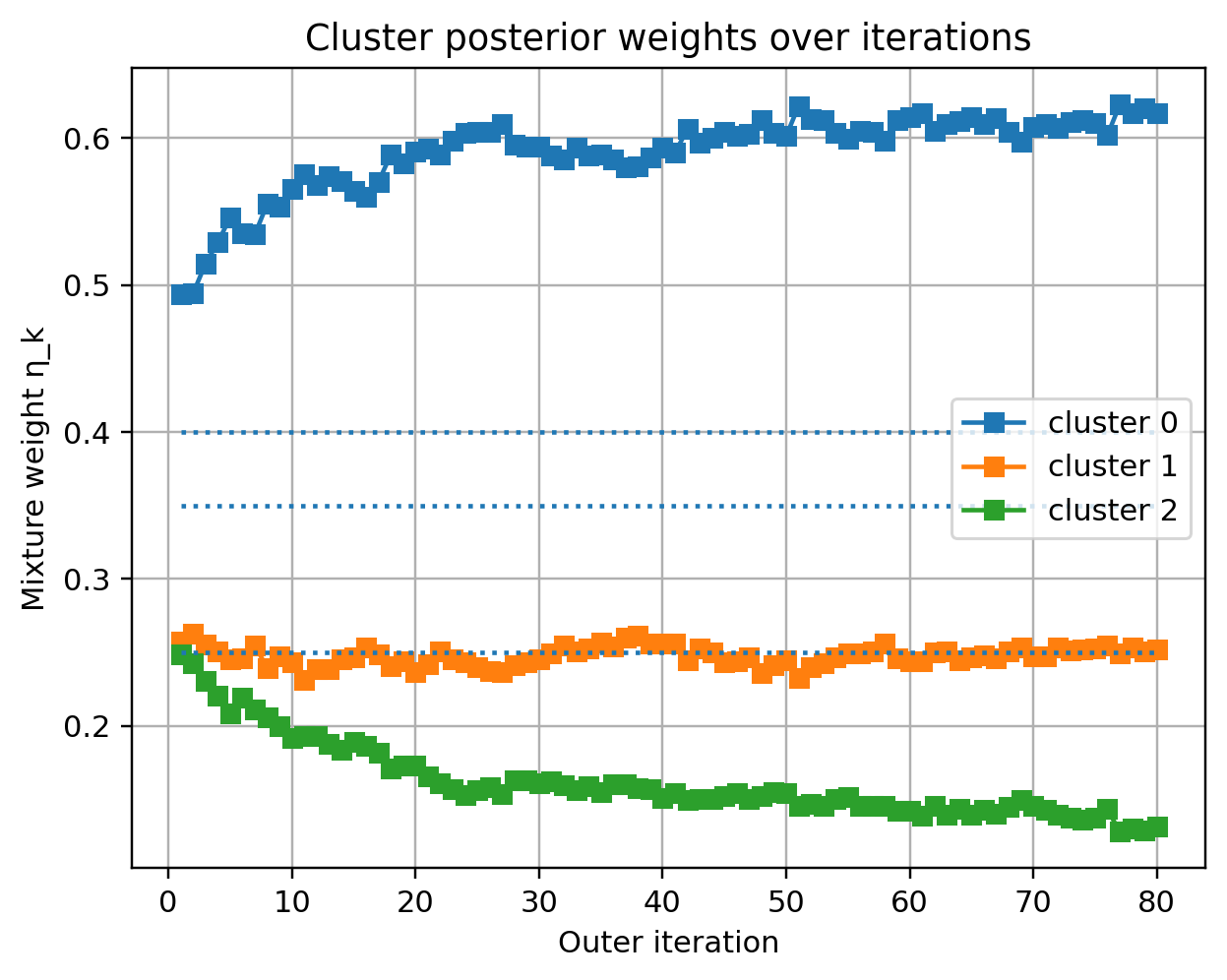}
  \caption{Clusterless}
\end{subfigure}\hfill
\begin{subfigure}{0.23\textwidth}
  \includegraphics[width=\linewidth]{Results/chemistry_exp/maobo-mix_plot_cluster_mixture_weights.png}
  \caption{Hybrid}
\end{subfigure}
\caption{Mixture weight trajectory for PET production process. Clusterless (left) collapses to a dominant component, while Hybrid (right) quickly identifies and stabilizes near the correct mode proportions.}
\label{fig:cluster_mixture_weights}
\end{figure}

\begin{figure}
\centering
\begin{subfigure}{0.23\textwidth}
  \includegraphics[width=\linewidth]{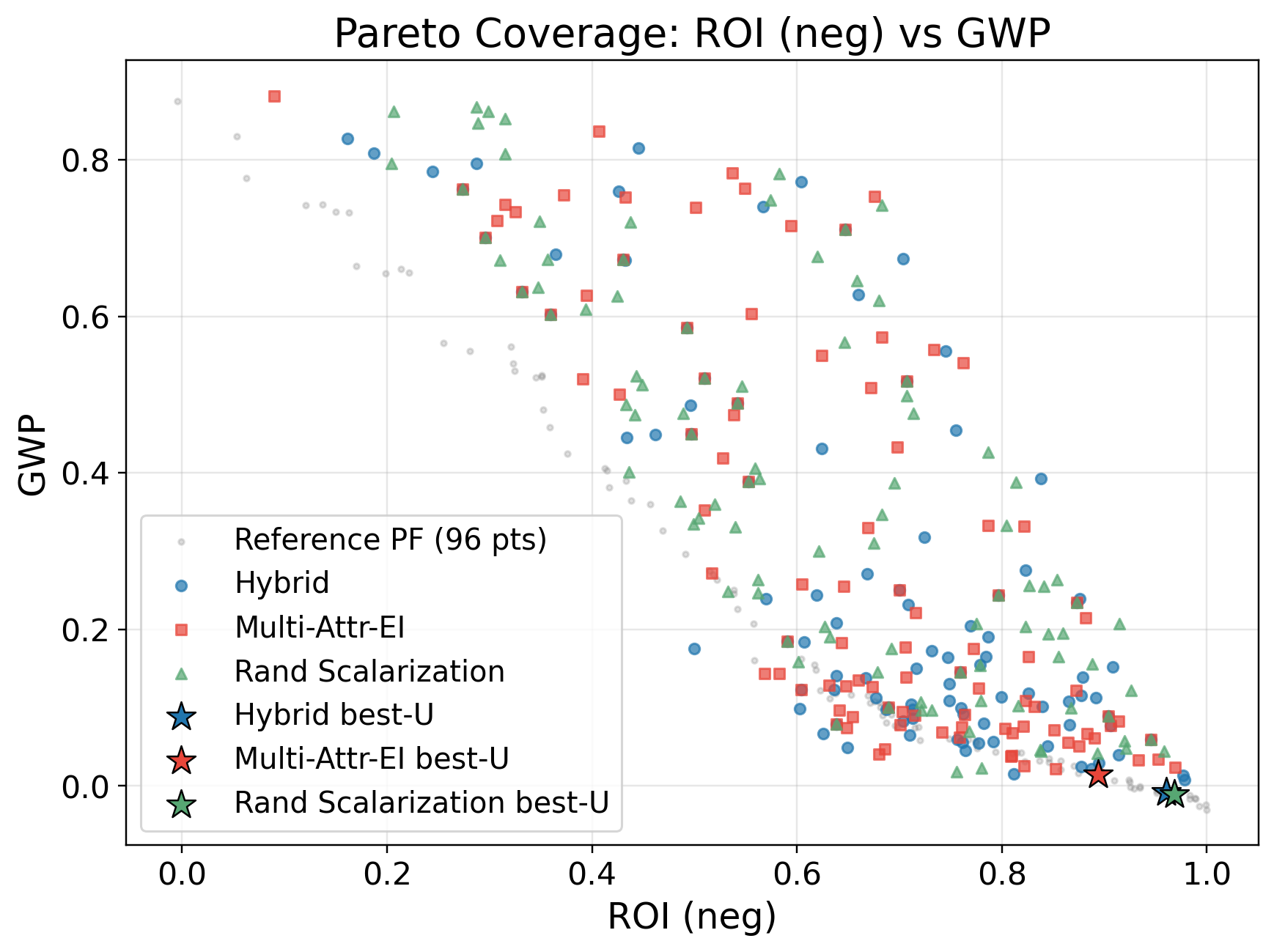}
  \caption{ROI vs GWP}
\end{subfigure} \hfill
\begin{subfigure}{0.23\textwidth}
  \includegraphics[width=\linewidth]{Results/ParetoFrontPlots/pareto_coverage_chemistry_obj0_obj6.png}
  \caption{ROI vs Toxicity}
\end{subfigure}
\caption{Pareto coverage in objective space. Grey points-reference Pareto front; colored points - evaluated solutions; stars - final best-utility solution under the true mixture preference. Efficient trade-offs lie along the grey frontier. Hybrid achieves broader frontier coverage and identifies final solutions closest to the mixture-preferred Pareto region.}
\label{fig:pareto_front}
\end{figure}

\section{Results and Analysis}

\subsection{Additional Preference-Recovery Diagnostics}

Simple regret measures optimization performance under the true mixture
utility, but does not reveal whether the inferred preference model
recovers the underlying archetype structure. We therefore report
additional diagnostics for mixture calibration, archetype recovery,
and localization on the Pareto front.

\paragraph{Mixture-weight recovery:}
Figure~\ref{fig:cluster_mixture_weights} compares the evolution of the
posterior mixture weights on the PET process-design case study. The
Clusterless policy increasingly concentrates probability on a single
component, indicating mode averaging and weak identification of minority
archetypes. In contrast, Hybrid maintains three non-negligible components
and stabilizes substantially closer to the true mixture proportions
$(0.40,0.35,0.25)$. The minority component remains slightly
underestimated, reflecting the difficulty of recovering less frequent
archetypes from a finite number of comparisons.

\paragraph{Localization on the Pareto front:}
Figure~\ref{fig:pareto_front} projects the evaluated PET designs onto the
ROI--GWP and ROI--Toxicity planes. Different archetypes may favour distinct
regions of the Pareto front, which aggregate regret does not expose.
Both Hybrid and the unimodal baseline sample broadly along the projected
efficient frontier, but Hybrid produces a final best-utility solution
slightly closer to the region preferred under the true mixture utility.
These projections provide qualitative evidence that mixture-aware
preference learning improves localization in decision-relevant regions,
rather than merely reducing scalar regret.

\paragraph{Archetype recovery and mode collapse:}
Figure~\ref{fig:metric5} reports component-wise $L_1$ errors after
Hungarian alignment with the true archetypes. Under Clusterless querying,
only one inferred component improves substantially while the remaining
components stay at high error, indicating mode collapse. Inter queries
reduce errors across multiple components by selecting comparisons that
differentiate competing archetypes, although non-monotone adjustments
remain as the inferred component alignment changes during learning.

\paragraph{Mixture calibration across query policies:}
Figure~\ref{fig:metric3} compares inferred and true mixture weights after
alignment. Random querying assigns excessive mass to a single component,
while Intra primarily refines the currently dominant archetype and does
not reliably calibrate the global mixture. Inter querying produces the
closest correspondence between inferred and true component masses,
supporting its intended role in archetype identification.

\begin{figure}
\centering
\begin{subfigure}{0.23\textwidth}
  \includegraphics[width=\linewidth]{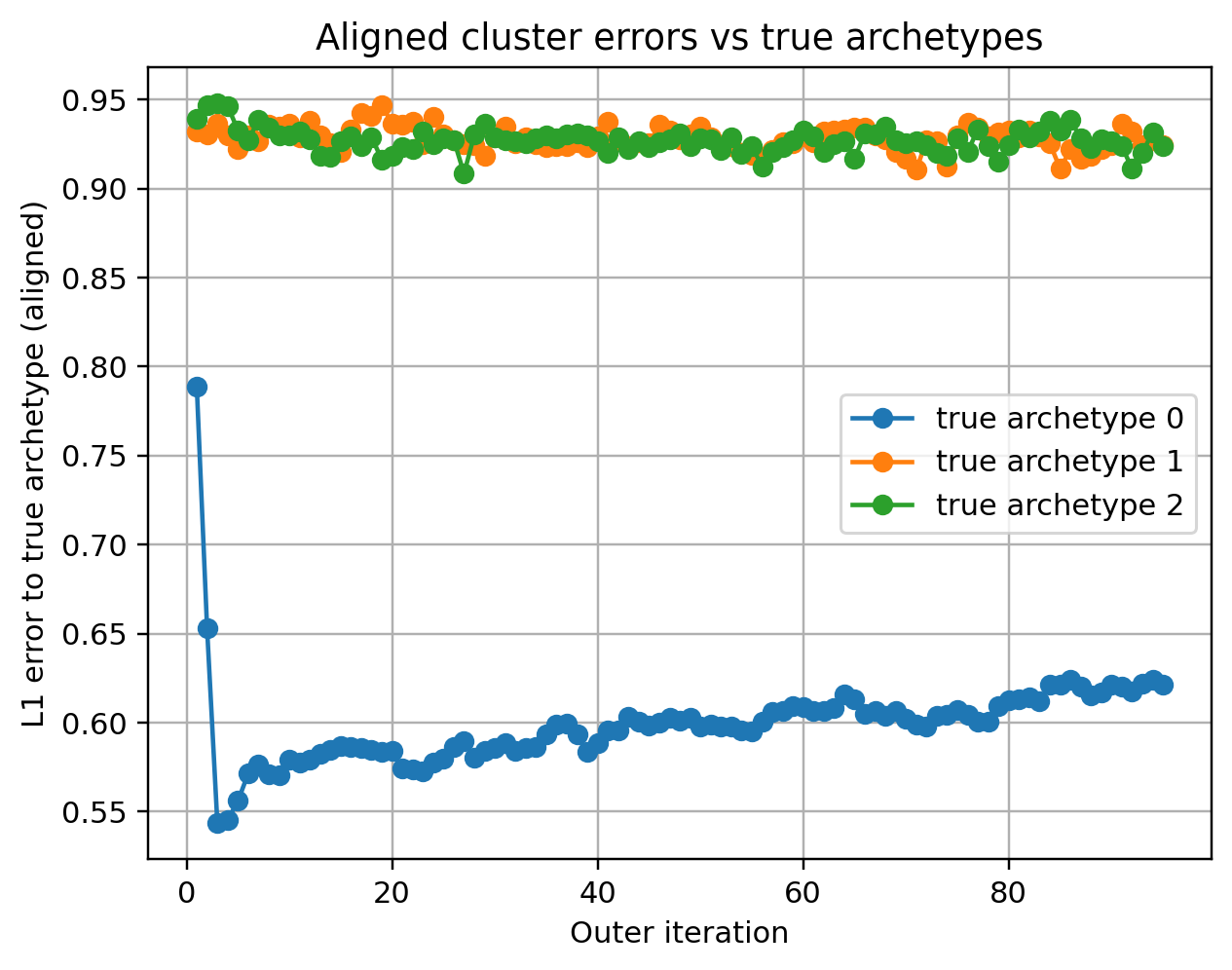}
  \caption{Clusterless}
\end{subfigure}\hfill
\begin{subfigure}{0.23\textwidth}
  \includegraphics[width=\linewidth]{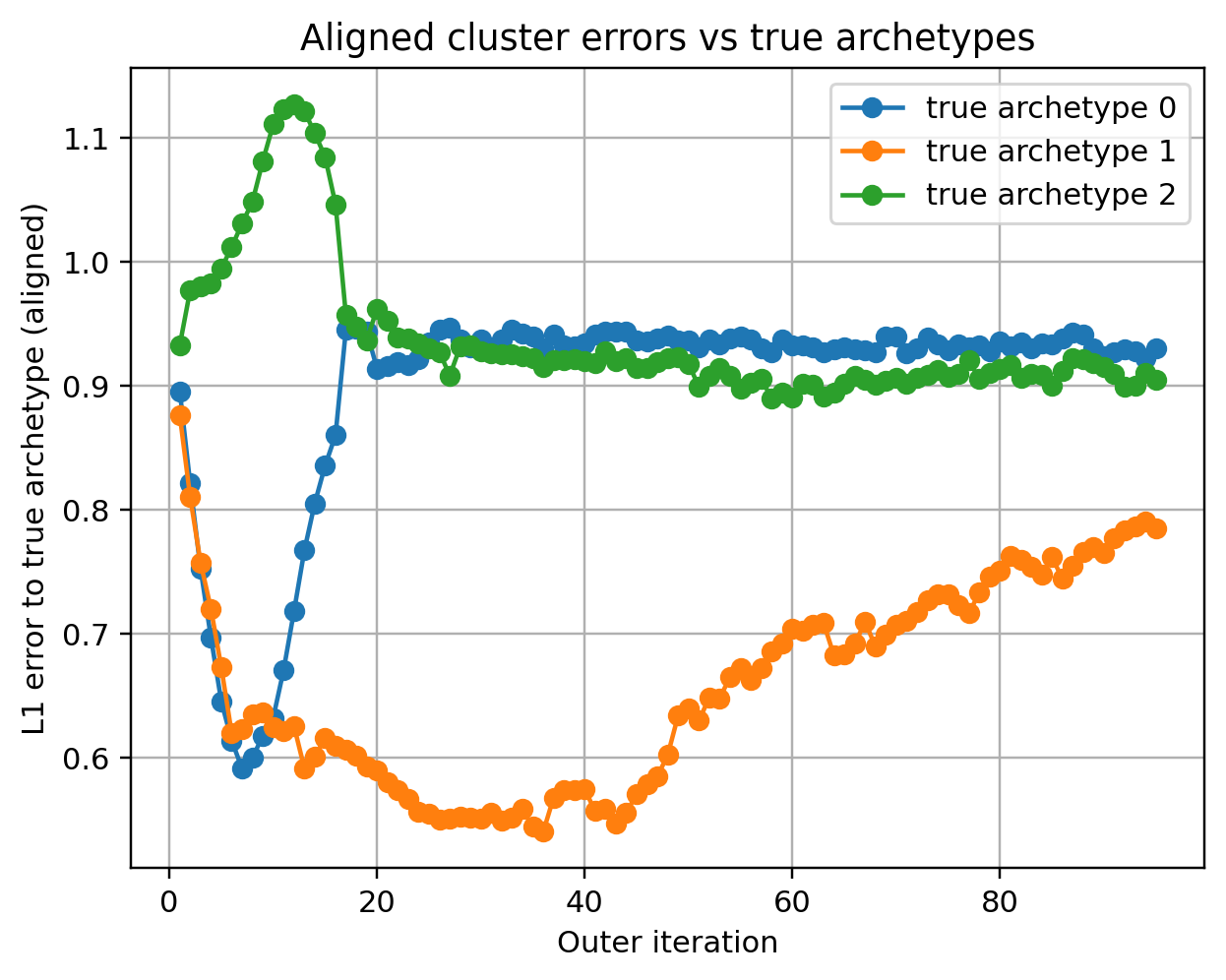}
  \caption{Inter}
\end{subfigure}
\caption{$L_1$ error trajectories per inferred component for a persistent user on DTLZ. After Hungarian alignment (lower is better), Inter queries prevent mode collapse and rapidly reduce error for at least one archetype, while Clusterless largely collapses, with only one component improving and others remaining high}
\label{fig:metric5}
\end{figure}

\begin{figure*}
\centering
\begin{subfigure}{0.33\textwidth}
  \includegraphics[width=\linewidth]{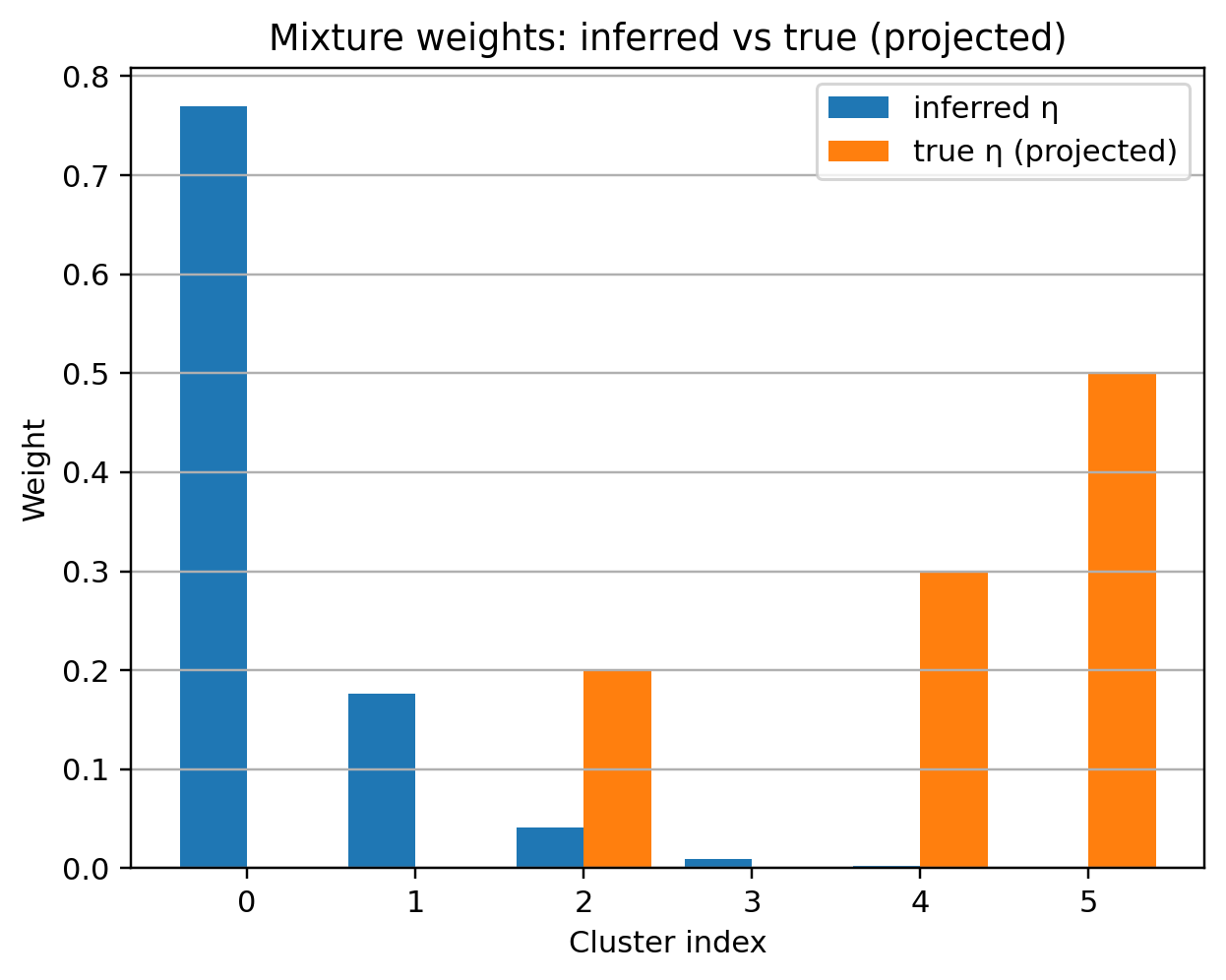}
  \caption{Random}
\end{subfigure}\hfill
\begin{subfigure}{0.33\textwidth}
  \includegraphics[width=\linewidth]{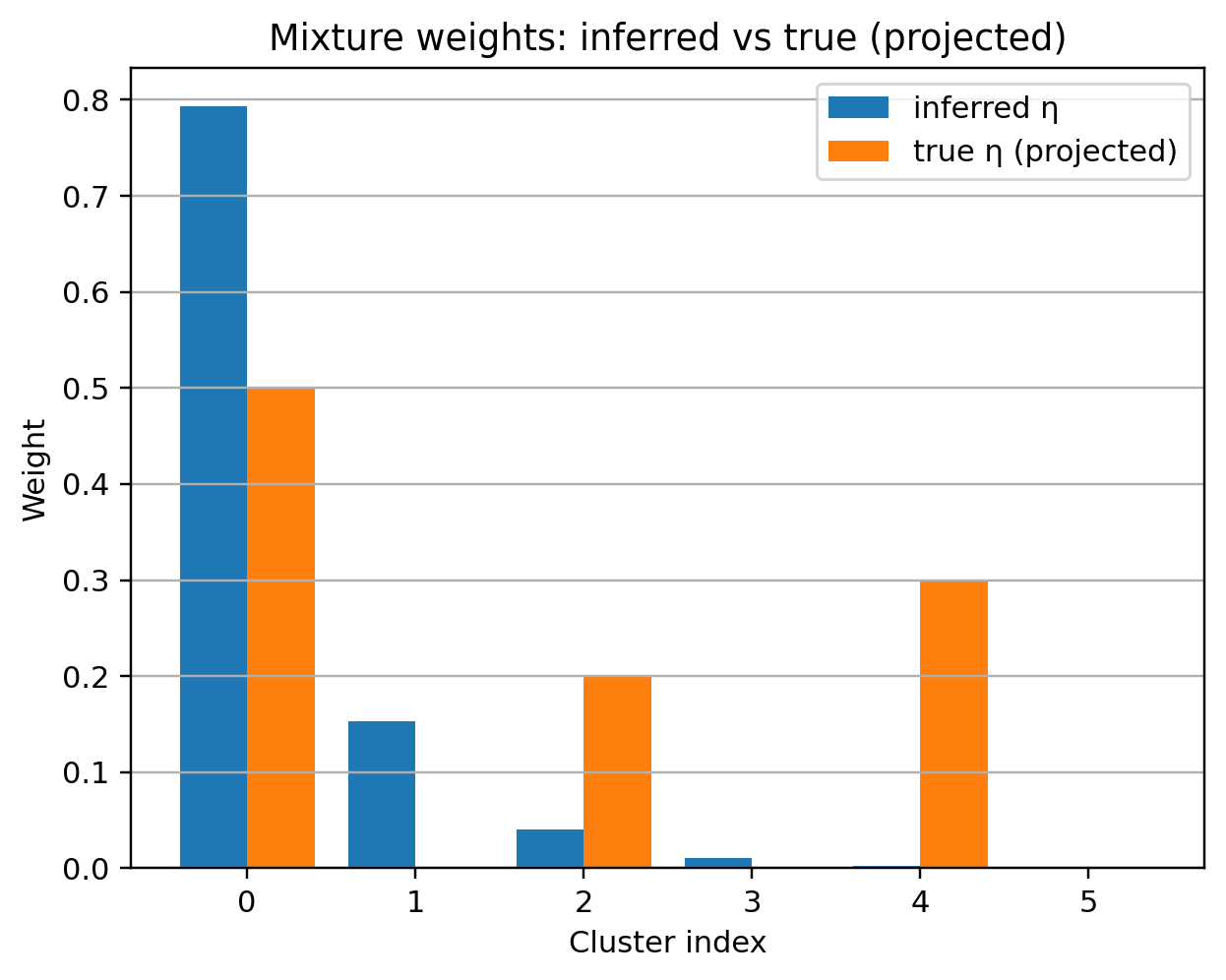}
  \caption{Intra}
\end{subfigure}\hfill
\begin{subfigure}{0.33\textwidth}
  \includegraphics[width=\linewidth]{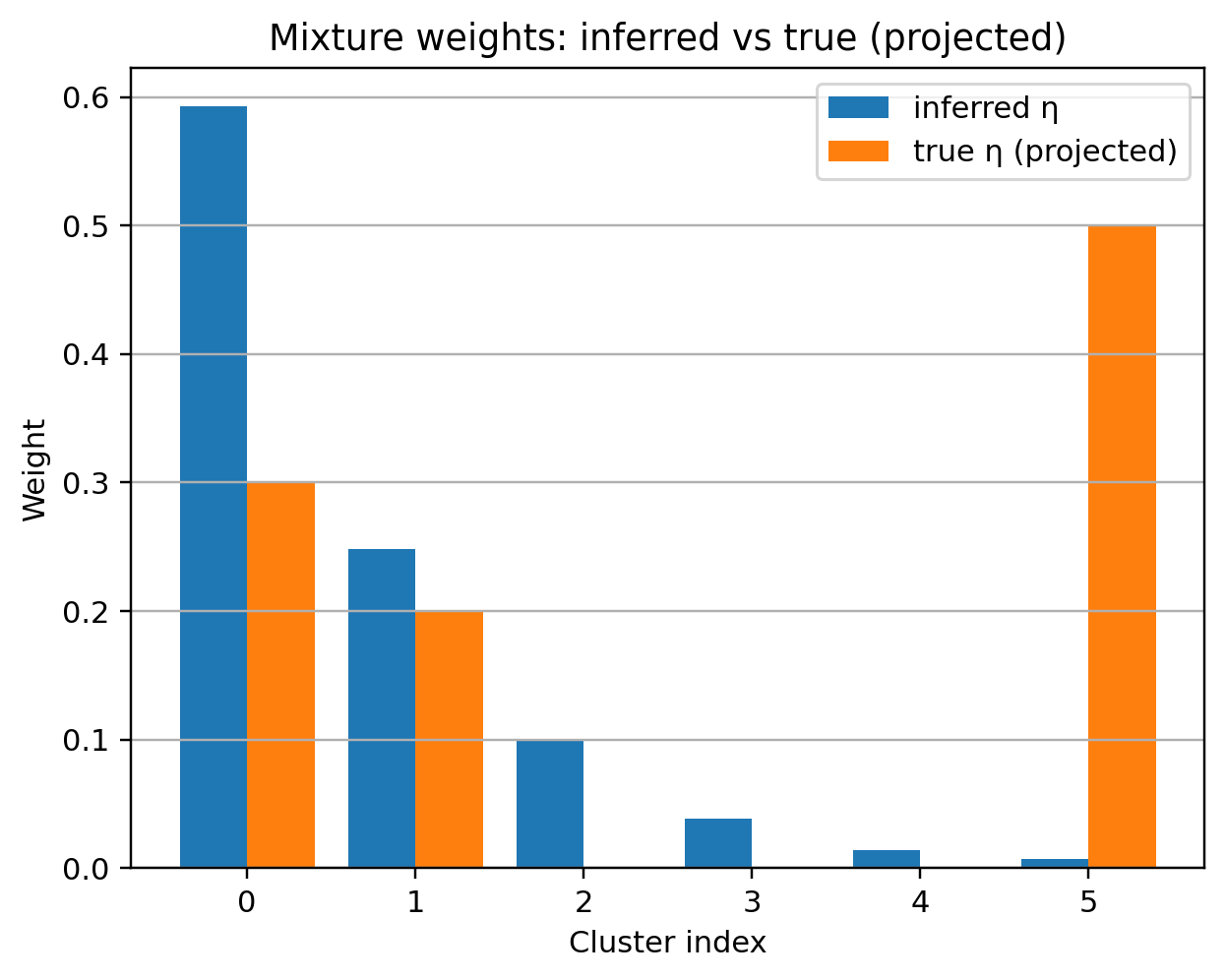}
  \caption{Inter}
\end{subfigure}\hfill
\caption{Mixture-weight calibration for persistent user for DTLZ dataset. Inter queries best match the true mixture, while random collapses to one mode and intra largely ignores 
$\eta$. Blue bars are inferred $\hat{\eta}$, orange bars are projected true $\eta^*$
via Hungarian alignment. (a) Random: blue mass on one index and orange-only bars → mode collapse. (b) Intra: blue focuses on one component → poor allocation across modes. (c) Inter: clear blue–orange alignment → best identification and calibration across modes}
\label{fig:metric3}
\end{figure*}

\subsection{Robustness Analysis}

We evaluate the robustness of the proposed framework under four challenging settings: homogeneous preferences, utility misspecification, varying degrees of preference heterogeneity, and misspecified mixture complexity.

\begin{figure}
\centering
\begin{subfigure}{0.23\textwidth}
  \includegraphics[width=\linewidth]{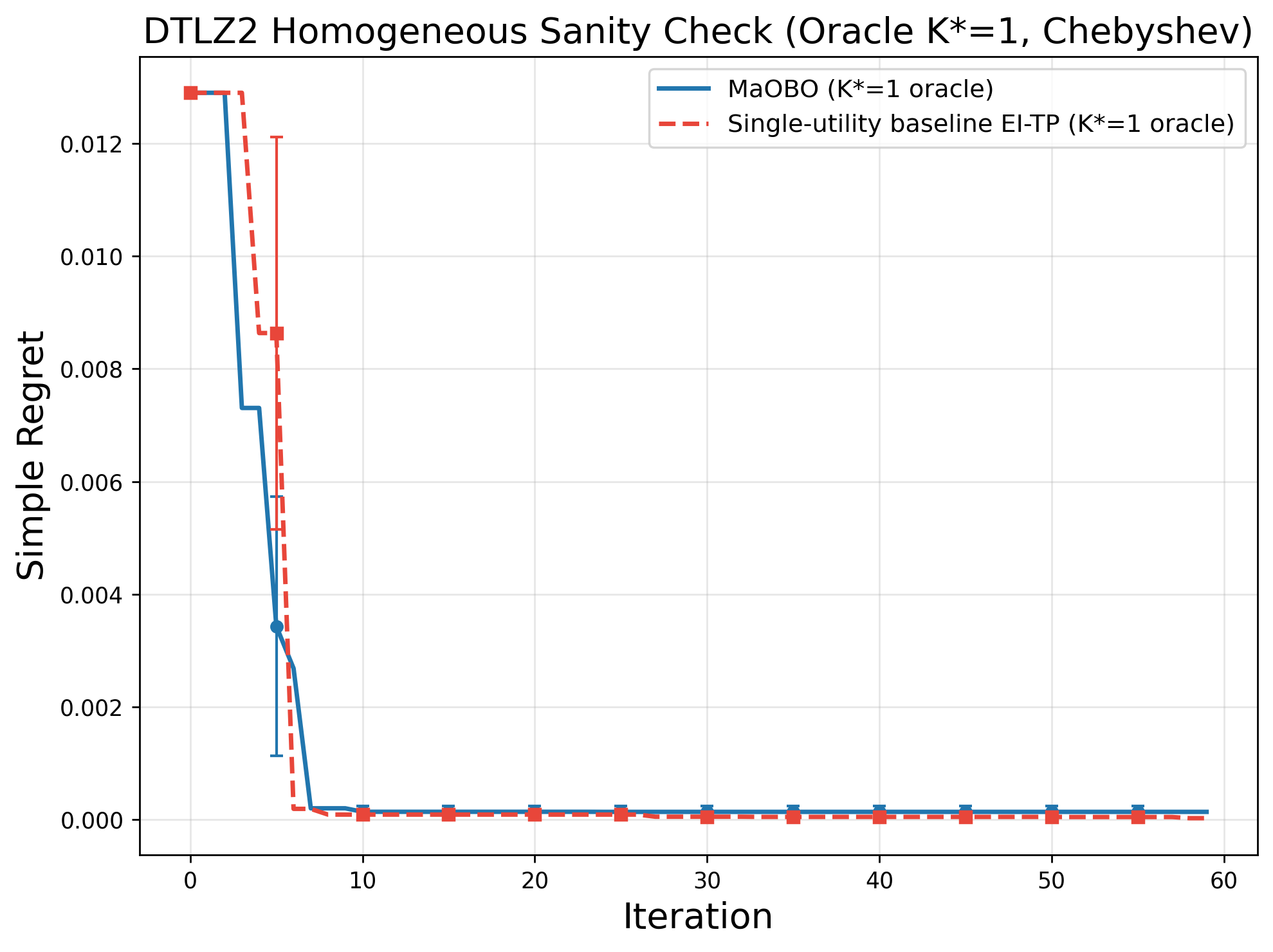}
  \caption{DTLZ}
\end{subfigure}\hfill
\begin{subfigure}{0.23\textwidth}
  \includegraphics[width=\linewidth]{Results/ExtraPlots/wfg_homog_oracle_matched.png}
  \caption{WFG}
\end{subfigure}
\vspace{-2mm}
\caption{Homogenous Oracle ($K^*=1$)}
\label{fig:robustness_homogeneous_oracle}
\vspace{-5mm}
\end{figure}

\begin{figure}
\centering
\begin{subfigure}{0.23\textwidth}
  \includegraphics[width=\linewidth]{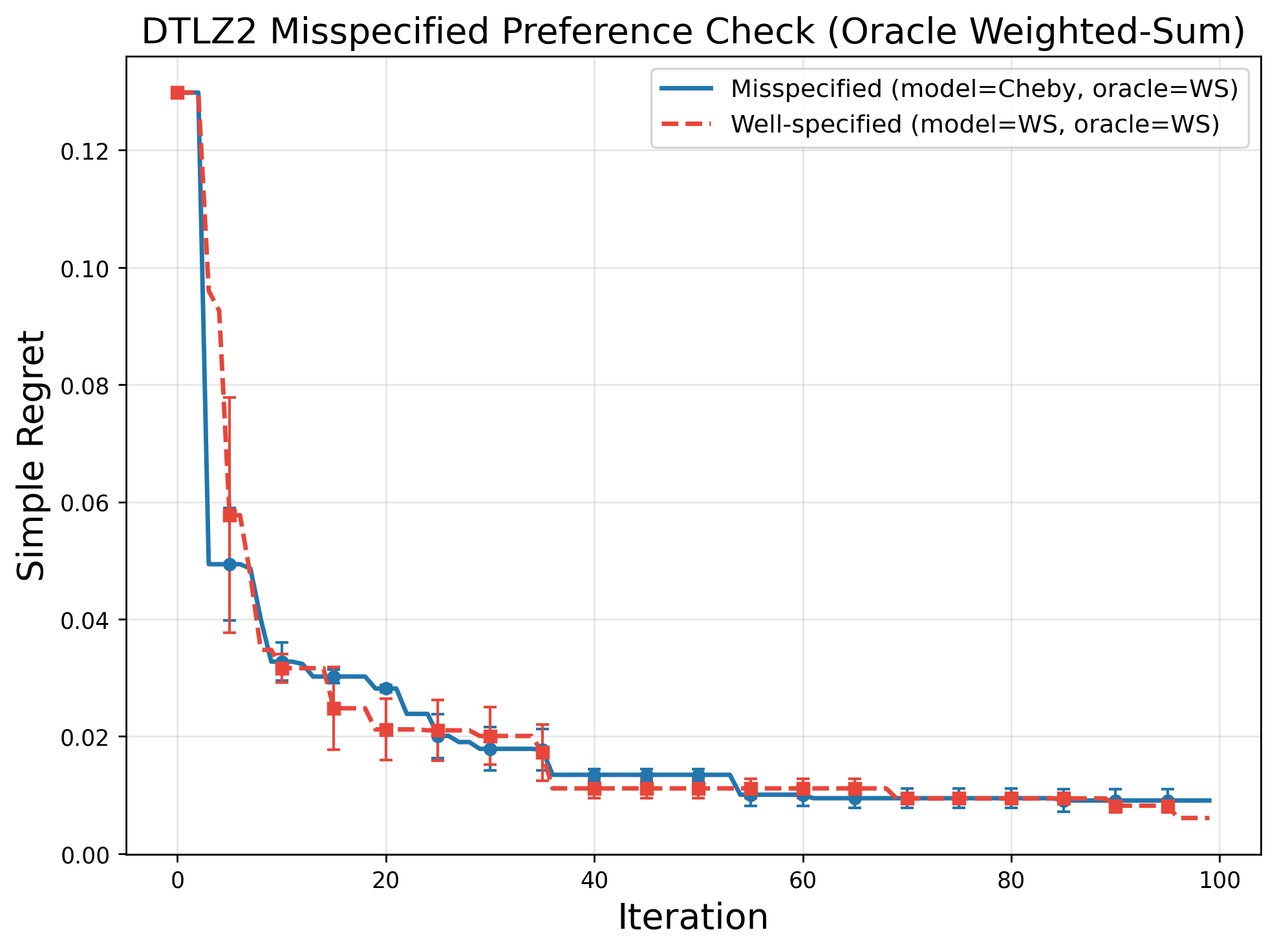}
  \caption{DTLZ}
\end{subfigure}\hfill
\begin{subfigure}{0.23\textwidth}
  \includegraphics[width=\linewidth]{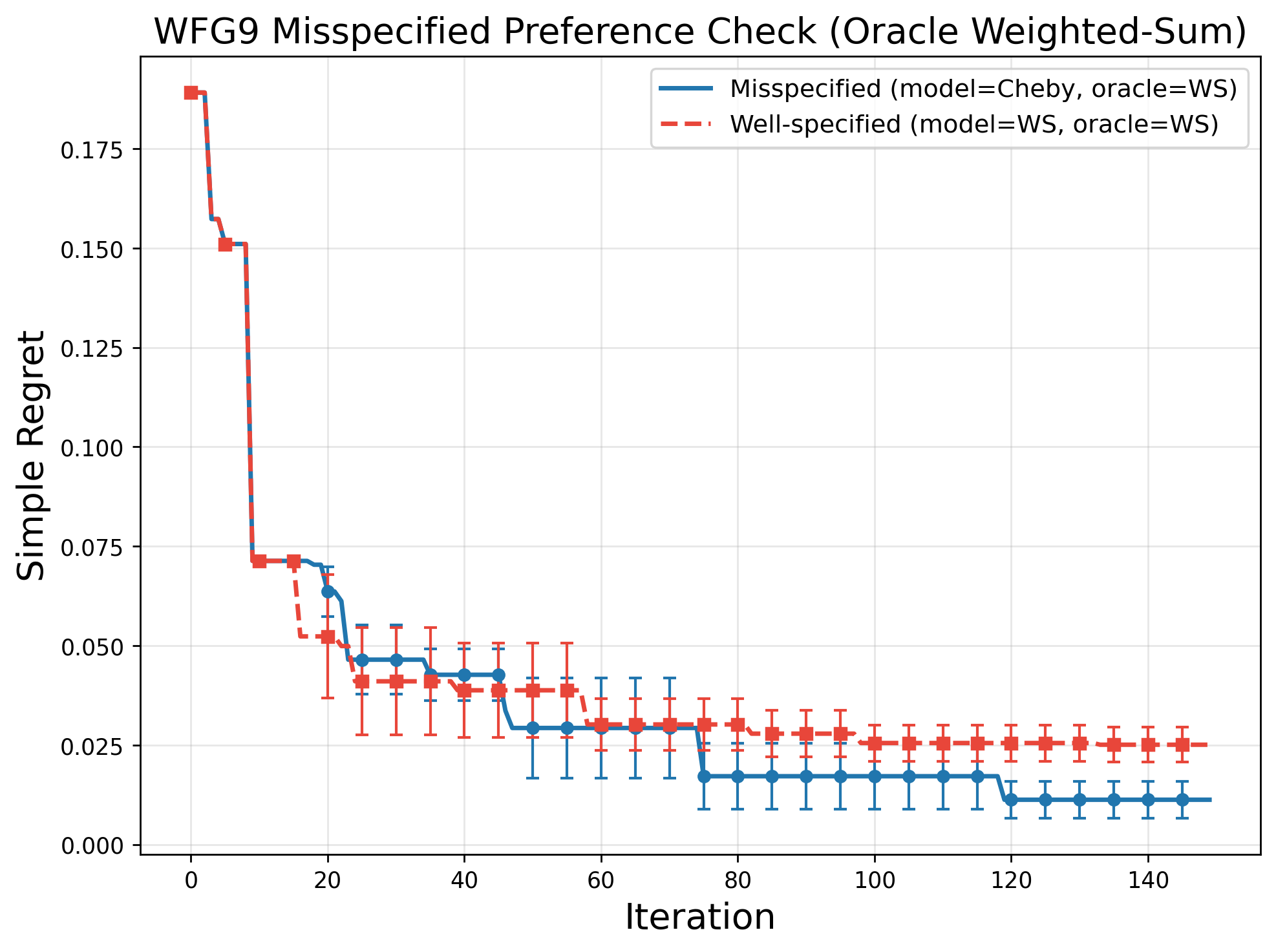}
  \caption{WFG}
\end{subfigure}
\vspace{-2mm}
\caption{Utility Misspecification}
\label{fig:robustness_utility_misspecification}
\vspace{-5mm}
\end{figure}

\begin{figure}
\centering
\begin{subfigure}{0.23\textwidth}
  \includegraphics[width=\linewidth]{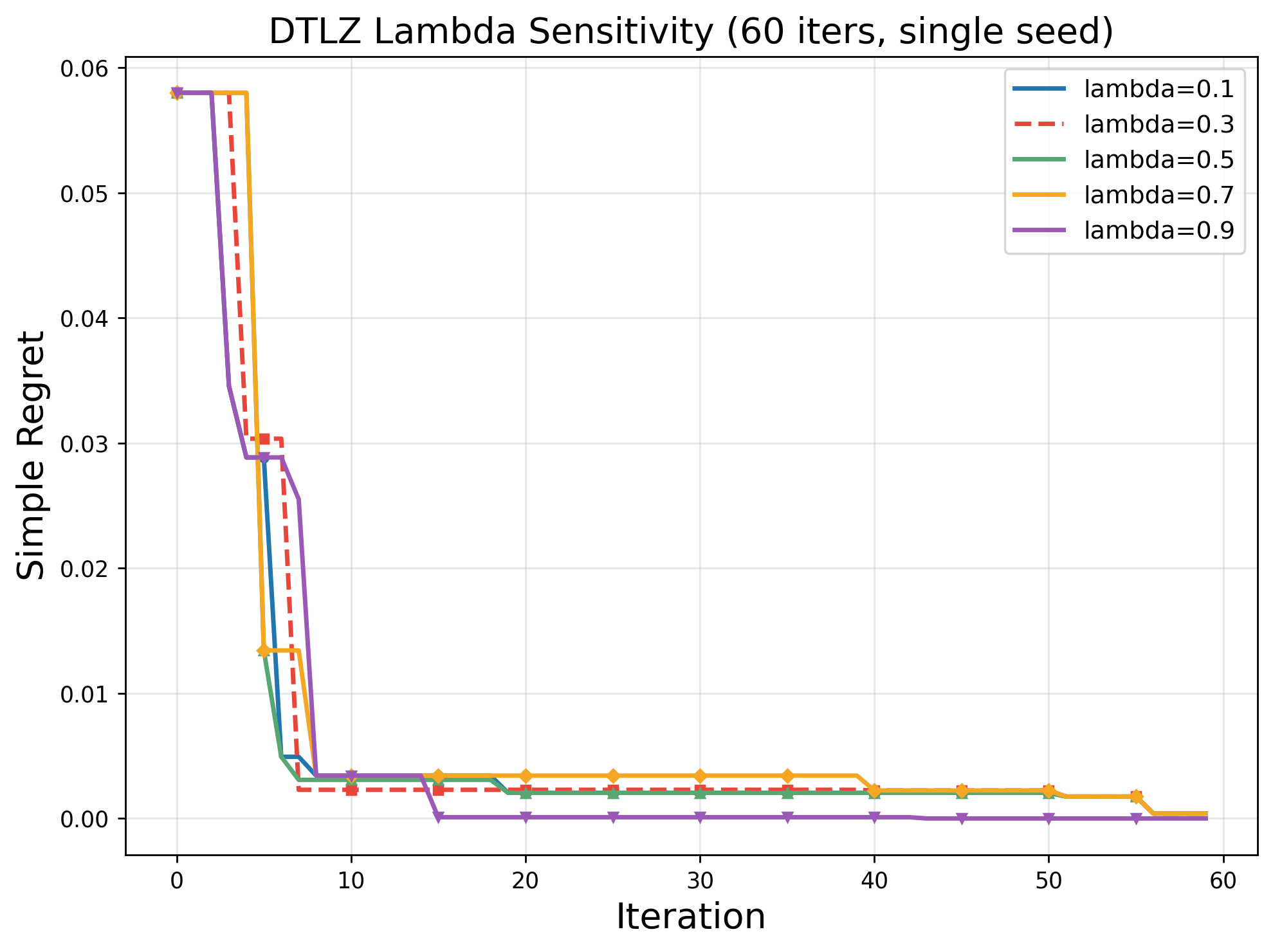}
  \caption{DTLZ Sensitivity}
\end{subfigure}\hfill
\begin{subfigure}{0.23\textwidth}
  \includegraphics[width=\linewidth]{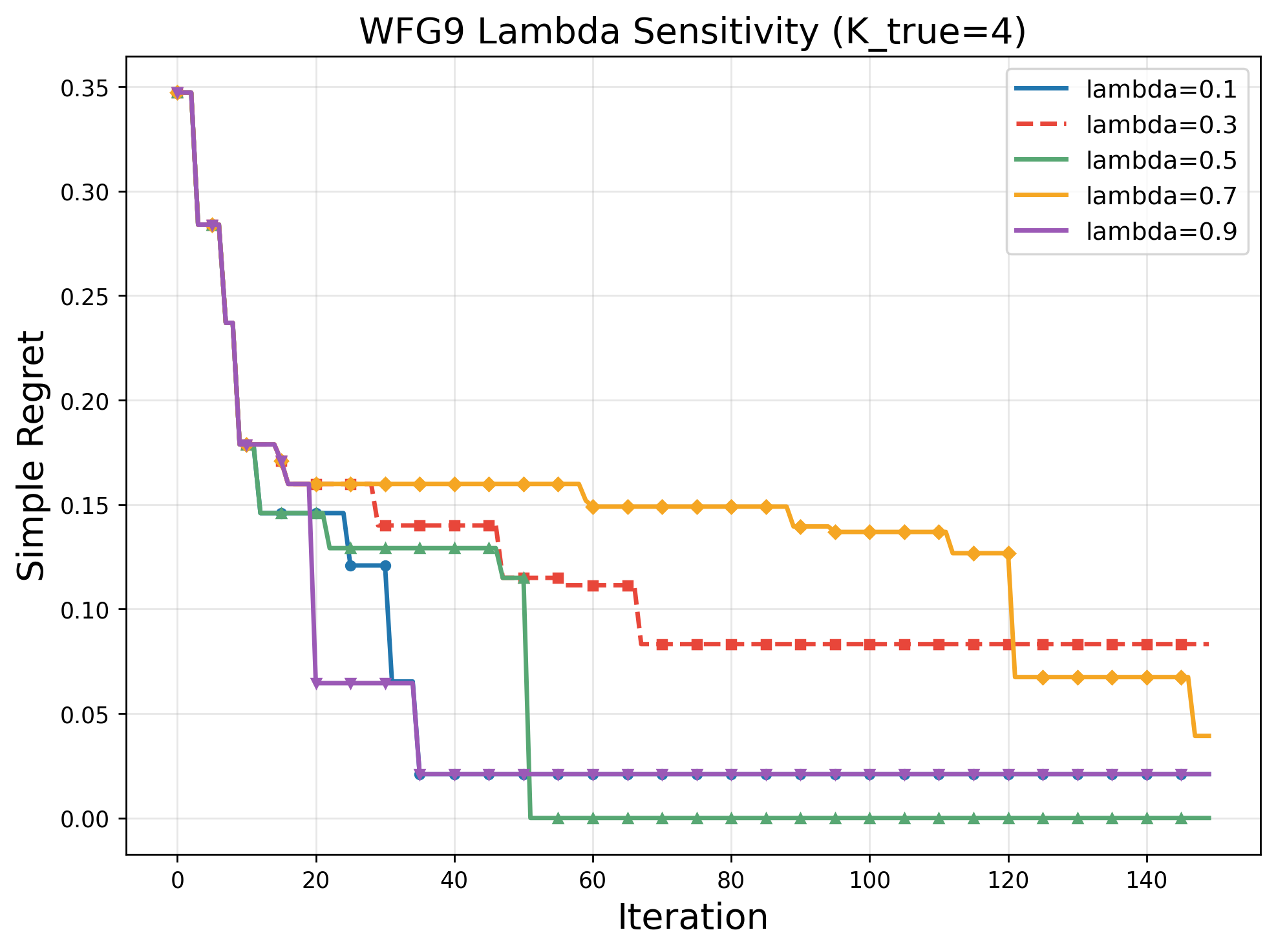}
  \caption{WFG}
\end{subfigure}
\vspace{-2mm}
\caption{Lambda Sensitivity}
\label{fig:robustness_sensitivity_lambda}
\vspace{-5mm}
\end{figure}

\begin{figure}
\centering
\begin{subfigure}{0.23\textwidth}
  \includegraphics[width=\linewidth]{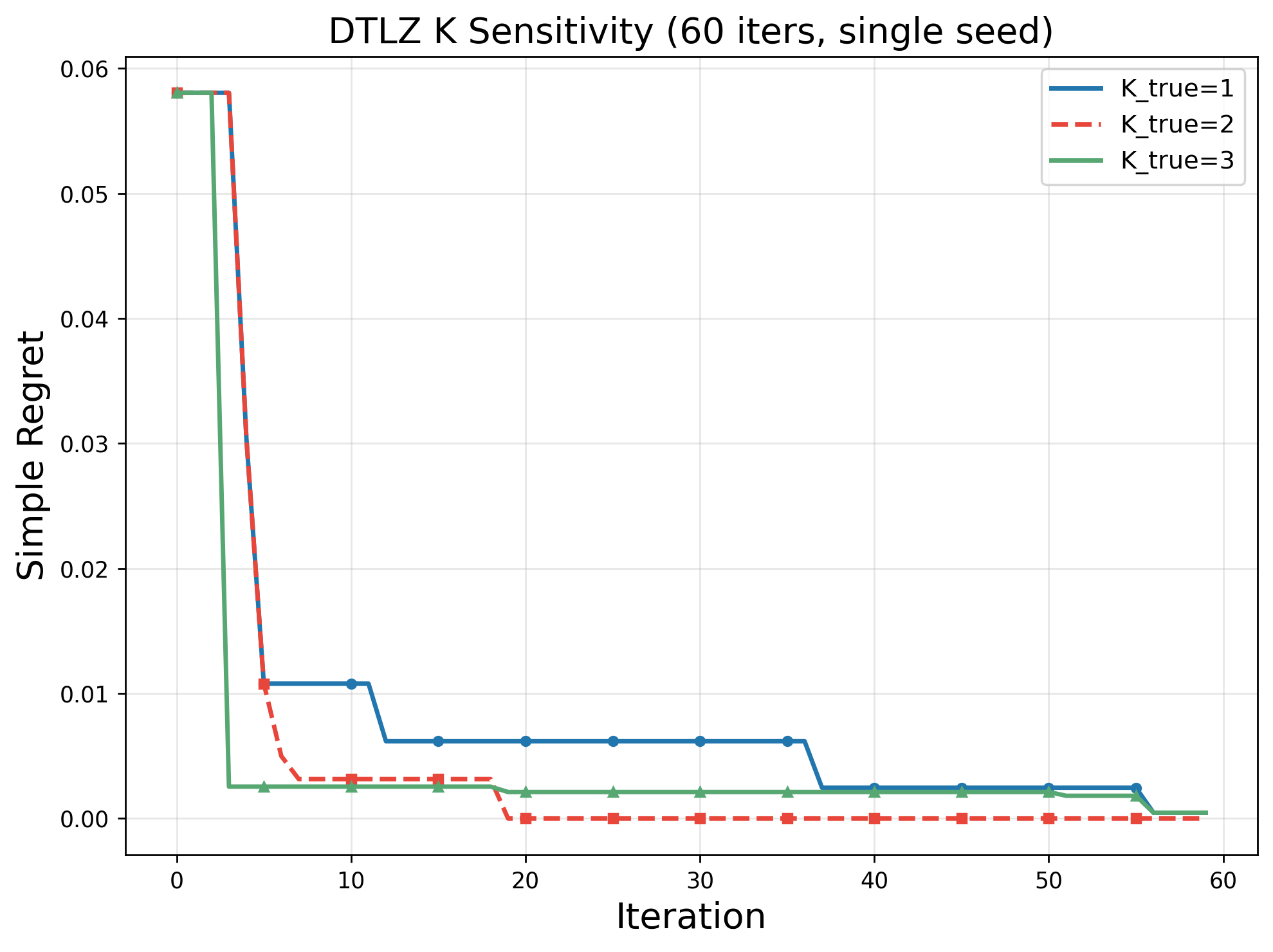}
  \caption{DTLZ}
\end{subfigure}\hfill
\begin{subfigure}{0.23\textwidth}
  \includegraphics[width=\linewidth]{Results/ExtraPlots/wfg_k_sensitivity.png}
  \caption{WFG}
\end{subfigure}
\vspace{-2mm}
\caption{Scalability with increasing preference heterogeneity}
\label{fig:robustness_sensitivity_modes}
\vspace{-5mm}
\end{figure}

\begin{figure}
\centering
\begin{subfigure}{0.23\textwidth}
  \includegraphics[width=\linewidth]{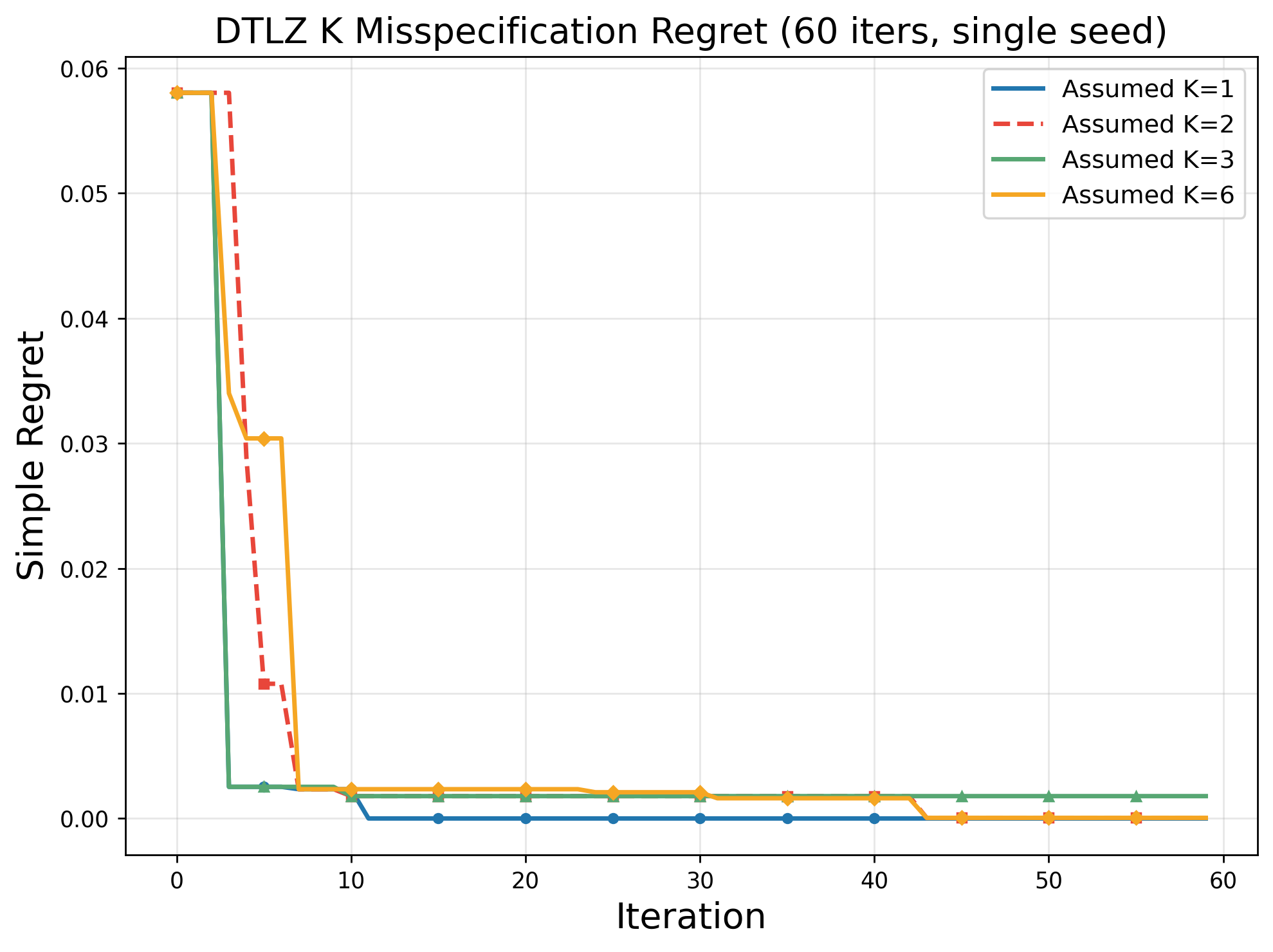}
  \caption{DTLZ}
\end{subfigure}\hfill
\begin{subfigure}{0.23\textwidth}
  \includegraphics[width=\linewidth]{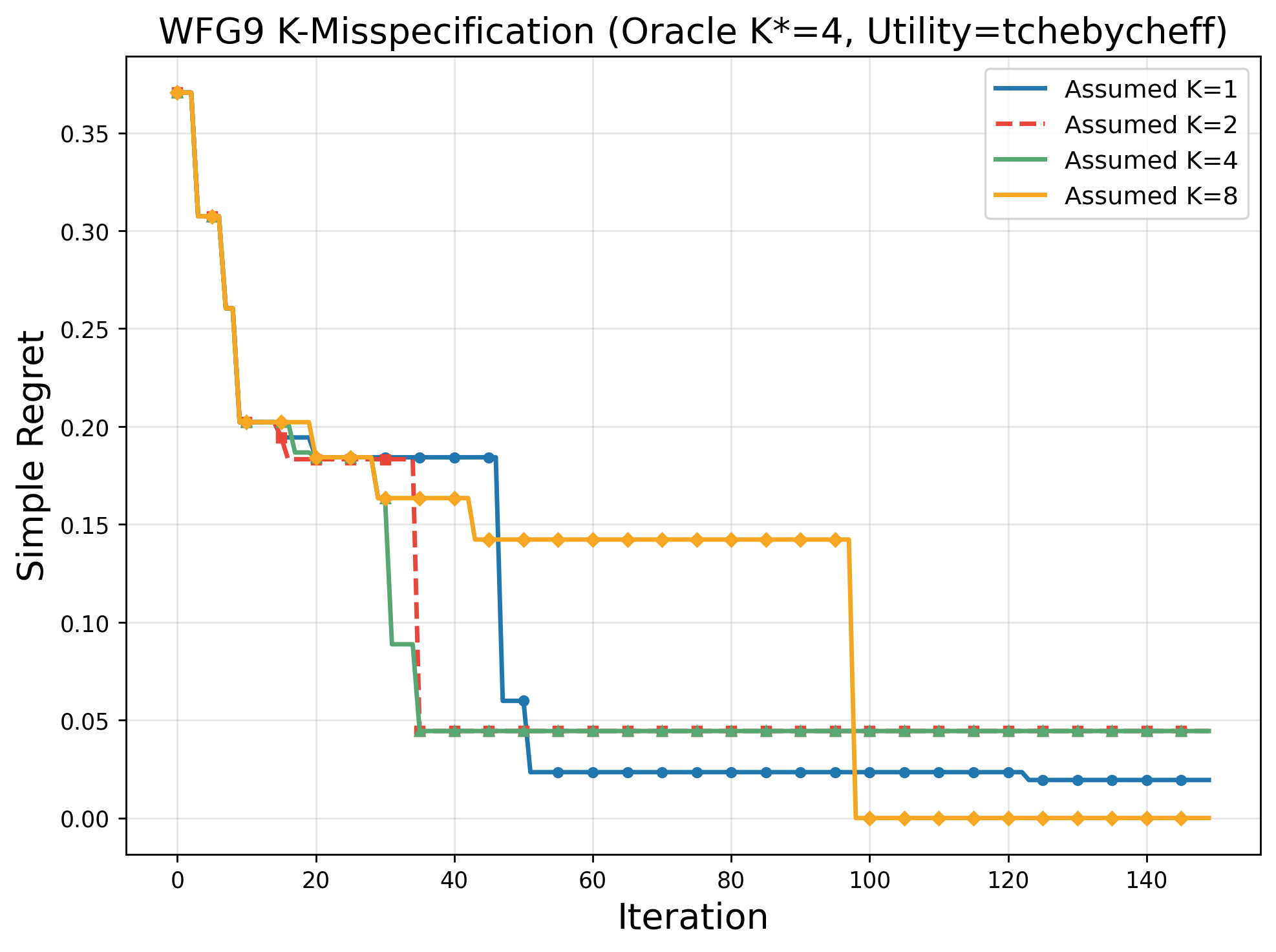}
  \caption{WFG}
\end{subfigure}
\vspace{-2mm}
\caption{Sensitivity to misspecified preference modes}
\label{fig:robustness_sensitivity_pref_missp}
\vspace{-5mm}
\end{figure}

\paragraph{Homogeneous preferences:}
Figure~\ref{fig:robustness_homogeneous_oracle} considers the degenerate case where all decision makers share a single preference archetype ($K^*=1$). As expected, the specialized single-utility baseline performs best because the true preference model is unimodal. Nevertheless, the proposed mixture-based framework remains competitive on both DTLZ2 and WFG, indicating that the additional flexibility introduced by the mixture model does not substantially degrade optimization performance when heterogeneous preferences are absent.

\paragraph{Utility misspecification:}
Figure~\ref{fig:robustness_utility_misspecification} evaluates robustness when the optimization model assumes a different utility representation from the oracle generating preferences (Chebyshev model versus weighted-sum oracle). Performance remains nearly identical to the well-specified setting across both benchmarks, suggesting that the proposed method is not overly dependent on the exact functional form of the utility model. Instead, it primarily exploits the relative preference structure revealed through pairwise comparisons.

\paragraph{Sensitivity to the hybrid query parameter:}
Figure~\ref{fig:robustness_sensitivity_lambda} studies the influence of the hybrid query parameter $\lambda$, which balances inter-mode exploration against intra-mode refinement. Performance is relatively stable across a broad range of values, while intermediate settings ($\lambda \approx 0.3$--$0.5$) consistently achieve the lowest regret. Very small $\lambda$ values focus excessively on local utility refinement before identifying the correct archetype, whereas very large values spend too many queries distinguishing modes at the expense of optimization.

\paragraph{Scalability with increasing preference heterogeneity:}
Figure~\ref{fig:robustness_sensitivity_modes} investigates increasing numbers of latent preference archetypes. As $K_{\mathrm{true}}$ increases, convergence naturally becomes slower because more preference information is required to distinguish among competing archetypes. Despite this increased identification complexity, the final regret remains nearly unchanged, demonstrating that the framework scales gracefully with growing preference diversity.

\paragraph{Misspecified number of preference archetypes:}
Figure~\ref{fig:robustness_sensitivity_pref_missp} evaluates robustness when the assumed number of latent preference modes differs from the true oracle ($K^*=4$). Moderate under-specification produces only a small degradation in performance because similar archetypes are effectively merged. Conversely, over-specification primarily slows convergence, as additional components receive little posterior probability before being discarded during inference. In both cases, the final optimization performance remains close to that of the correctly specified model, indicating robustness to moderate errors in selecting the mixture size.

\section{Theoretical Proofs}
In this section, we detail the proof of the theorem stated in the main paper. We state model setup and assumptions. Then we prove Lipschitz properties of the utility. Further, we discuss surrogate mismatch bound. Further, we combine all the pieces for the regret decomposition theorem. 

\subsection*{Model Assumptions}
We consider $m$ objectives to be \emph{minimized}. For any outcome vector $y\in\mathbb R^m$
and weights $w\in\Delta^{m-1}$, define the Chebyshev utility
\begin{equation}
\label{eq:neg-cheb}
U(y,w)\coloneqq -\min_{1\le j\le m}\frac{y_j}{w_j}.
\end{equation}
Thus larger utility is preferred. We assume all weights are bounded away from the simplex
boundary: there exists $c_w>0$ such that $w_j\ge c_w$ for all $j$ for every true archetype
$w_k^\star$ and every estimated archetype $\hat w_{k,t}$.

Let the true latent preferences be a mixture of $K_\star$ archetypes with weights
$\eta^\star\in\Delta^{K_\star-1}$ and archetypes $\{w_k^\star\}_{k=1}^{K_\star}\subset\Delta^{m-1}$.
Define the true mixture utility at design $x\in\mathcal X$ by
\begin{equation}
\label{eq:true-mix-util}
\mathcal U^\star(x)\coloneqq \sum_{k=1}^{K_\star}\eta_k^\star\,U\!\big(f(x),w_k^\star\big),
\qquad
\mathcal U^\star \coloneqq \sup_{x\in\mathcal X} \mathcal U^\star(x).
\end{equation}

The algorithm fits one GP per objective $f_j$ and maintains posterior mean
$\mu_{t-1,j}(x)$ and standard deviation $\sigma_{t-1,j}(x)$ after $t-1$ evaluations.
Let $\mu_{t-1}(x)\in\mathbb R^m$ be the vector of means. At round $t$, the preference model outputs estimates $\hat\eta_t\in\Delta^{K-1}$
and $\{\hat w_{k,t}\}_{k=1}^K$. Define the surrogate mixture utility used for acquisition
\begin{equation}
\label{eq:surrogate-mix-util}
\widehat{\mathcal U}_t(x)\coloneqq \sum_{k=1}^{K}\hat\eta_{k,t}\,U\!\big(\mu_{t-1}(x),\hat w_{k,t}\big).
\end{equation}
Assume the selected point $x_t$ is $\varepsilon_t$-optimal for the surrogate:
\begin{equation}
\label{eq:eps-opt}
\widehat{\mathcal U}_t(x_t)\ge \sup_{x\in\mathcal X}\widehat{\mathcal U}_t(x)-\varepsilon_t.
\end{equation}

\begin{assumption}[Bounded Objectives]
$\|f(x)\|_\infty \le B_y \quad \forall x\in\mathcal X$
\end{assumption}

\paragraph{Alignment and estimation errors}
Let $\pi_t$ be a permutation (e.g., Hungarian matching) aligning estimated archetypes
to true archetypes. Pad $\eta^\star$ with zeros if $K>K_\star$. Define
\begin{equation}
\label{eq:deltas}
\Delta_t^{w}\coloneqq \sum_{k=1}^{K_\star}\eta_k^\star\,
\|\hat w_{\pi_t(k),t}-w_k^\star\|_1,
\qquad
\Delta_t^{\eta}\coloneqq \|\hat\eta_t-\eta^\star\|_1.
\end{equation}

\begin{lemma}[Lipschitzness of Chebyshev utility]
\label{lem:lipschitz-neg-cheb}
Let $U$ be as in \eqref{eq:neg-cheb} and assume $w_j,w'_j\ge c_w>0$ for all $j$.
Then for any $y,y'\in\mathbb R^m$ and any such $w,w'\in\Delta^{m-1}$,
\begin{align}
\big|U(y,w)-U(y',w)\big|
&\le \frac{1}{c_w}\,\|y-y'\|_\infty,
\label{eq:Lip-y}\\
\big|U(y,w)-U(y,w')\big|
&\le \frac{\|y\|_\infty}{c_w^2}\,\|w-w'\|_1.
\label{eq:Lip-w}
\end{align}
\end{lemma}

\begin{proof}
For \eqref{eq:Lip-y}, fix $w$ and define $g_j(y)\coloneqq -y_j/w_j$. Each $g_j$ is
$(1/c_w)$-Lipschitz under $\|\cdot\|_\infty$ because
$|g_j(y)-g_j(y')|=|y_j-y'_j|/w_j\le \|y-y'\|_\infty/c_w$.
Since $U(y,w)=\min_j g_j(y)$ is the pointwise minimum of $(1/c_w)$-Lipschitz functions,
it is itself $(1/c_w)$-Lipschitz, proving \eqref{eq:Lip-y}.

For \eqref{eq:Lip-w}, fix $y$ and write $U(y,w)=\min_j h_j(w)$ with
$h_j(w)\coloneqq -y_j/w_j$. For any $j$,
\begin{equation}
\begin{aligned}
|h_j(w)-h_j(w')|
&=
|y_j|
\left|
\frac{1}{w_j}-\frac{1}{w'_j}
\right| \\
&=
\frac{|y_j|\,|w_j-w'_j|}{w_jw'_j} \\
&\le
\frac{\|y\|_\infty}{c_w^2}\,|w_j-w'_j|.
\end{aligned}
\end{equation}
Therefore
\begin{equation}
\begin{aligned}
|U(y,w)-U(y,w')|
&\le
\max_j |h_j(w)-h_j(w')| \\
&\le
\frac{\|y\|_\infty}{c_w^2}
\max_j|w_j-w'_j|
\le
\frac{\|y\|_\infty}{c_w^2}\|w-w'\|_1 .
\end{aligned}
\end{equation}
which proves \eqref{eq:Lip-w}.
\end{proof}

\subsection*{GP confidence}

We use a standard high-probability confidence bound for each objective GP.
Assume sub-Gaussian observation noise and conditions under which the standard GP-UCB
confidence event holds for each objective (e.g., \cite{srinivas2010gaussian}).

\begin{assumption}[Per-objective GP confidence event]
\label{ass:gp-conf}
There exists a nondecreasing sequence $\{\beta_t\}$ such that with probability at least
$1-\delta$, simultaneously for all $t\ge 1$, all $x\in\mathcal X$, and all $j\in\{1,\dots,m\}$,
\begin{equation}
\label{eq:gp-per-obj}
|f_j(x)-\mu_{t-1,j}(x)|
\le
\sqrt{\beta_t}\,\sigma_{t-1,j}(x).
\end{equation}
\end{assumption}

Under \eqref{eq:gp-per-obj}, we have for all $x$,
\begin{equation}
\label{eq:vector-gp}
\|f(x)-\mu_{t-1}(x)\|_\infty
\le
\sqrt{\beta_t}\,\|\sigma_{t-1}(x)\|_\infty.
\end{equation}

\subsection*{Main result: simple regret decomposition}

Define the simple regret under the \emph{true} mixture utility:
\begin{equation}
\label{eq:simple-regret}
R_T \coloneqq \mathcal U^\star - \max_{t\le T}\mathcal U^\star(x_t).
\end{equation}

\begin{theorem}[Simple regret decomposition]
\label{thm:regret-decomp-safe}
On the GP confidence event in Assumption~\ref{ass:gp-conf}, for any $T\ge 1$,
\begin{equation}
\label{eq:regret-decomp}
\begin{aligned}
R_T \le\;&
\underbrace{\frac{2}{c_w}\min_{t\le T}
\|f(x_t)-\mu_{t-1}(x_t)\|_\infty}_{\textnormal{Objective GP error}}
+
\underbrace{\frac{2B_y}{c_w^2}\min_{t\le T}\Delta_t^{w}}_{\textnormal{Archetype error}}
\\[1mm]
&
+
\underbrace{\frac{2B_y}{c_w}\min_{t\le T}\Delta_t^{\eta}}_{\textnormal{Mixture-weight error}}
+
\underbrace{\max_{t\le T}\varepsilon_t}_{\textnormal{Acquisition optimization error}}.
\end{aligned}
\end{equation}
Moreover, still on the same event, the objective GP term admits an existence-type rate:
\begin{equation}
\label{eq:gp-rate-existence}
\begin{aligned}
\min_{t\le T}\|f(x_t)-\mu_{t-1}(x_t)\|_\infty
&\le
\frac{1}{T}\sum_{t=1}^T
\|f(x_t)-\mu_{t-1}(x_t)\|_\infty
\\
&\lesssim
\sqrt{\frac{\beta_T\,\gamma_T}{T}}.
\end{aligned}
\end{equation}
where $\gamma_T$ is a (max) information-gain term for the objective GPs (taking the maximum
over objectives if kernels differ), and $\lesssim$ hides constants depending only on the kernels and noise.
\end{theorem}

\begin{proof}
Let $x^\star\in\arg\max_{x\in\mathcal X}\mathcal U^\star(x)$. Fix any round $t\ge 1$.
Add and subtract the surrogate utility \eqref{eq:surrogate-mix-util} at $x^\star$ and $x_t$:
\begin{equation}
\label{eq:gp-rate-existence}
\begin{aligned}
\min_{t\le T}\|f(x_t)-\mu_{t-1}(x_t)\|_\infty
&\le
\frac{1}{T}\sum_{t=1}^T
\|f(x_t)-\mu_{t-1}(x_t)\|_\infty
\\
&\lesssim
\sqrt{\frac{\beta_T\,\gamma_T}{T}}.
\end{aligned}
\end{equation}
where $\gamma_T$ denotes the maximum information gain of the objective GPs
(maximized over objectives if different kernels are used), and $\lesssim$
hides constants depending only on the kernels and observation noise.
By $\varepsilon_t$-optimality \eqref{eq:eps-opt}, the middle term is at most $\varepsilon_t$.
Therefore,
\begin{equation}
\label{eq:two-mismatch}
\mathcal U^\star(x^\star)-\mathcal U^\star(x_t)
\le
\big|\mathcal U^\star(x^\star)-\widehat{\mathcal U}_t(x^\star)\big|
+
\big|\mathcal U^\star(x_t)-\widehat{\mathcal U}_t(x_t)\big|
+
\varepsilon_t.
\end{equation}
It remains to bound $|\mathcal U^\star(x)-\widehat{\mathcal U}_t(x)|$ for a generic $x$.

\textbf{Decomposition}
Using \eqref{eq:true-mix-util} and \eqref{eq:surrogate-mix-util},
\[
\mathcal U^\star(x)-\widehat{\mathcal U}_t(x)
=
\sum_{k=1}^{K_\star}\eta_k^\star U(f(x),w_k^\star)
-
\sum_{k=1}^{K}\hat\eta_{k,t} U(\mu_{t-1}(x),\hat w_{k,t}).
\]
Insert and subtract $\sum_{k=1}^{K_\star}\eta_k^\star U(\mu_{t-1}(x),w_k^\star)$ and
$\sum_{k=1}^{K_\star}\eta_k^\star U(\mu_{t-1}(x),\hat w_{\pi_t(k),t})$, then apply triangle inequality:
\begin{equation}
\label{eq:mismatch-decomp}
\begin{aligned}
\big|\mathcal U^\star(x)-\widehat{\mathcal U}_t(x)\big|
\le\;&
\sum_{k=1}^{K_\star}\eta_k^\star
\big|U(f(x),w_k^\star)-U(\mu_{t-1}(x),w_k^\star)\big|
\\
&+
\sum_{k=1}^{K_\star}\eta_k^\star
\big|U(\mu_{t-1}(x),w_k^\star)
-U(\mu_{t-1}(x),\hat w_{\pi_t(k),t})\big|
\\
&+
\left|
\sum_{k=1}^{K}
(\eta_k^\star-\hat\eta_{k,t})
\right.
\\
&\qquad\left.
\times\,U(\mu_{t-1}(x),\hat w_{k,t})
\right|.
\end{aligned}
\end{equation}

\textbf{Bounding terms using Lemma~\ref{lem:lipschitz-neg-cheb}}
For the first line, apply Lipschitzness in $y$ \eqref{eq:Lip-y}:
\[
\big|U(f(x),w_k^\star)-U(\mu_{t-1}(x),w_k^\star)\big|
\le
\frac{1}{c_w}\,\|f(x)-\mu_{t-1}(x)\|_\infty.
\]
Summing over $k$ with weights $\eta_k^\star$ gives
\begin{equation}
\label{eq:termA}
\sum_{k=1}^{K_\star}\eta_k^\star
\big|U(f(x),w_k^\star)-U(\mu_{t-1}(x),w_k^\star)\big|
\le
\frac{1}{c_w}\,\|f(x)-\mu_{t-1}(x)\|_\infty.
\end{equation}

For the second line, apply Lipschitzness in $w$ \eqref{eq:Lip-w} with $\| \mu_{t-1}(x)\|_\infty\le B_y$:
\[
\big|U(\mu_{t-1}(x),w_k^\star)-U(\mu_{t-1}(x),\hat w_{\pi_t(k),t})\big|
\le
\frac{B_y}{c_w^2}\,\|\hat w_{\pi_t(k),t}-w_k^\star\|_1.
\]
Thus the second line is bounded by $(B_y/c_w^2)\Delta_t^w$.

For the third line, use $|U(\mu_{t-1}(x),\hat w_{k,t})|
\le \|\mu_{t-1}(x)\|_\infty/c_w \le B_y/c_w$ to obtain
\begin{equation}
\label{eq:termC}
\left|\sum_{k=1}^{K}(\eta_k^\star-\hat\eta_{k,t})\,U(\mu_{t-1}(x),\hat w_{k,t})\right|
\le
\frac{B_y}{c_w}\,\|\hat\eta_t-\eta^\star\|_1
=
\frac{B_y}{c_w}\,\Delta_t^\eta.
\end{equation}

Combining \eqref{eq:mismatch-decomp}, \eqref{eq:termA} and \eqref{eq:termC}, we conclude that for all $x$,
\begin{equation}
\label{eq:mismatch-bound-final}
\big|\mathcal U^\star(x)-\widehat{\mathcal U}_t(x)\big|
\le
\frac{1}{c_w}\,\|f(x)-\mu_{t-1}(x)\|_\infty
+
\frac{B_y}{c_w^2}\Delta_t^w
+
\frac{B_y}{c_w}\Delta_t^\eta.
\end{equation}

\textbf{Convert to simple regret}
Apply \eqref{eq:mismatch-bound-final} at $x=x^\star$ and $x=x_t$ in \eqref{eq:two-mismatch}:
\begin{equation}
\label{eq:per-round-bound}
\begin{aligned}
\mathcal U^\star(x^\star)-\mathcal U^\star(x_t)
\le\;&
\frac{1}{c_w}
\sum_{z\in\{x^\star,x_t\}}
\|f(z)-\mu_{t-1}(z)\|_\infty
\\
&+
\frac{2B_y}{c_w^2}\Delta_t^w
+
\frac{2B_y}{c_w}\Delta_t^\eta
+
\varepsilon_t .
\end{aligned}
\end{equation}

Now take the minimum over $t\le T$ on the right-hand side. Since
$R_T = \min_{t\le T}\big(\mathcal U^\star(x^\star)-\mathcal U^\star(x_t)\big)$ and
$\max_{t\le T}\varepsilon_t$ upper bounds $\varepsilon_t$, we obtain \eqref{eq:regret-decomp}. On the GP confidence event \eqref{eq:gp-per-obj}, the bound
\[
\|f(x)-\mu_{t-1}(x)\|_\infty
\le
\sqrt{\beta_t}\,\|\sigma_{t-1}(x)\|_\infty
\]
holds for all $x\in\mathcal X$, including $x^\star$.

\paragraph{Existence-type GP rate}
On the confidence event \eqref{eq:gp-per-obj},
\[
\|f(x_t)-\mu_{t-1}(x_t)\|_\infty
\le
\sqrt{\beta_t}\,\|\sigma_{t-1}(x_t)\|_\infty
\le
\sqrt{\beta_T}\,\|\sigma_{t-1}(x_t)\|_\infty.
\]
Under standard variance-sum bounds for GP regression (e.g., Lemma X),
\[
\sum_{t=1}^T \|\sigma_{t-1}(x_t)\|_\infty^2
\le C\,\gamma_T,
\]
which implies
\[
\min_{t\le T} \|\sigma_{t-1}(x_t)\|_\infty
\le
\frac1T\sum_{t=1}^T \|\sigma_{t-1}(x_t)\|_\infty
\le
\sqrt{\frac{C\,\gamma_T}{T}}.
\]

\end{proof}

\section{Limitations and Future Work} The proposed method evaluates the proposed framework in simulated settings and on a real-world process design dataset, but does not include human-subject experiments. Due to time and resource constraints, validation with real decision makers was beyond the scope of the current study and remains an important direction for future work. We also plan to test the method on a broader range of real-world multi-objective datasets to assess generalization across domains. Finally, extending the framework to more adaptive mixture structures instead of stationary latent archetypes could further improve flexibility and practical applicability.



\end{document}